\documentclass[a4paper,11pt]{article}

\usepackage{hyperref}
\usepackage{natbib}
\usepackage{amsmath}
\usepackage{amssymb}
\usepackage{amsthm}
\usepackage[dvips]{graphicx}
\usepackage{caption}
\usepackage{subcaption}
\usepackage{color}
\usepackage[normalem]{ulem}  
\usepackage{cancel}  
\usepackage{bm}  
\usepackage{mathtools}  
\usepackage{tikz}  
\usepackage{bbm}  
\usetikzlibrary{arrows}
\usetikzlibrary{positioning}

\def\ci{\perp\!\!\!\perp}

\newcommand{\cov}{\mathrm{cov}}

\newcommand{\tr}{\mathrm{tr}}
\newcommand{\bs}{\boldsymbol}

\DeclareMathOperator*{\argmax}{arg\,max}

\newtheorem{lemma}{Lemma}
\newtheorem{theorem}{Theorem}
\newtheorem{corollary}{Corollary}

\newtheorem{mydef}{Definition}
\newtheorem{algorithm}{Algorithm}

\usepackage{color}

\title{Distributional Equivalence and Structure Learning for Bow-free~Acyclic~Path~Diagrams}
\author{Christopher Nowzohour\footnote{Corresponding Author}\\Seminar f\"ur Statistik, ETH Z\"urich\\nowzohour@stat.math.ethz.ch \and Marloes H.\ Maathuis\\Seminar f\"ur Statistik, ETH Z\"urich\\maathuis@stat.math.ethz.ch \and Robin J.\ Evans\footnote{Was supported by the institute for mathematical research (FIM) during a visit to ETH Z\"urich}\\Department of Statistics, University of Oxford\\evans@stats.ox.ac.uk \and Peter B\"uhlmann\\Seminar f\"ur Statistik, ETH Z\"urich\\buhlmann@stat.math.ethz.ch}

\begin{document}

\bibliographystyle{abbrvnat2}

\maketitle

\begin{abstract}
We consider the problem of structure learning for bow-free acyclic path diagrams (BAPs). BAPs can be viewed as a generalization of linear Gaussian DAG models that allow for certain hidden variables. We present a first method for this problem using a greedy score-based search algorithm. We also prove some necessary and some sufficient conditions for distributional equivalence of BAPs which are used in an algorithmic approach to compute (nearly) equivalent model structures. This allows us to infer lower bounds of causal effects. We also present applications to real and simulated datasets using our publicly available R-package.
\end{abstract}

\section{Introduction}
We consider learning the causal structure among a set of variables from observational data. In general, the data can be modelled with a structural equation model (SEM) over the observed and unobserved variables, which expresses each variable as a function of its direct causes and a noise term, where the noise terms are assumed to be mutually independent. The structure of the SEM can be visualized as a directed graph, with vertices representing variables and edges representing direct causal relationships. We assume the structure to be recursive (acyclic), which results in a directed acyclic graph (DAG). DAGs can be understood as models of conditional independence, and many structure learning algorithms use this to find all DAGs which are compatible with the observed conditional independencies \citep{ SpiPA93}. Often, however, not all relevant variables are observed. The resulting marginal distribution over the observed variables might still satisfy some conditional independencies, but in general these will not have a DAG representation \citep{ RicTS02}. Also, there generally are additional constraints resulting from the marginalization of some of the variables \citep{ EvaR16, ShpI14}.

In this paper we consider a model class which can accommodate certain hidden variables. Specifically, we assume that the graph over the observed variables is a bow-free acyclic path diagram (BAP). This means it can have directed as well as bidirected edges (with the directed part being acyclic), where the directed edges represent direct causal effects, and the bidirected edges represent hidden confounders. The bow-freeness condition means there cannot be both a directed and a bidirected edge between the same pair of variables. The BAP can be obtained from the underlying DAG over all (hidden and observed) variables via a latent projection operation \citep{ PeaJ00} (if the bow-freeness condition admits this). We furthermore assume a parametrization with linear structural equations and Gaussian noise, where two noise terms are correlated only if there is a bidirected edge between the two respective nodes. In certain situations, it is beneficial to consider this restricted class of hidden variable models, as it forms a middle ground between DAG models that don't allow any hidden variables and maximal ancestral graph (MAG) models \citep{RicTS02} that allow arbitrarily many and general hidden variables. Such a restricted model class, if not heavily misspecified, results in a smaller distributional equivalence class, and estimation is expected to be more accurate than for more general hidden variable methods like FCI \citep{ SpiPA93}, RFCI \citep{ ColDA12}, or FCI+ \citep{ ClaTA13}.

The goal of this paper is structure learning with BAPs, that is, finding the set of BAPs that best explains some observational data. Just like in other models, there is typically an equivalence class of BAPs that are statistically indistinguishable, so a meaningful structure search result should represent this equivalence class. We propose a penalized likelihood score that is greedily optimized and a heuristic algorithm (supported by some theoretical results) for finding equivalent models once an optimum is found. This method is the first of its kind for BAP models.

\subsection*{Example of a BAP}
Consider the DAG in Figure~\ref{fig:motivatingexample.dag}, where we observe variables $X_1, \ldots, X_4$, but do not observe $H_1, H_2, H_3$. The only (conditional) independency over the observed variables is $X_1 \ci X_3 ~|~ X_2$, which is also represented in the corresponding BAP in Figure~\ref{fig:motivatingexample.bap}. The parametrization of this BAP would be
\begin{align*}
  X_1 &= \epsilon_1 \\
  X_2 &= B_{21} X_1 + \epsilon_2 \\
  X_3 &= B_{32} X_2 + \epsilon_3 \\
  X_4 &= B_{43} X_3 + \epsilon_4
\end{align*}
with $(\epsilon_1, \epsilon_2, \epsilon_3, \epsilon_4)^T \sim \mathcal{N}(\mathbf{0}, \Omega)$ and
\begin{align*}
  \Omega =
    \begin{pmatrix}
      \Omega_{11} & 0 & 0 & 0 \\
      0 & \Omega_{22} & 0 & \Omega_{24} \\
      0 & 0 & \Omega_{33} & 0 \\
      0 & \Omega_{24} & 0 & \Omega_{44}
    \end{pmatrix}.
\end{align*}
Hence the model parameters in this case are $B_{21}$, $B_{32}$, $B_{43}$, $\Omega_{11}$, $\Omega_{22}$, $\Omega_{33}$, $\Omega_{44}$, and $\Omega_{24}$. An example of an acyclic path diagram that is not bow-free is shown in Figure~\ref{fig:motivatingexample.mag}.

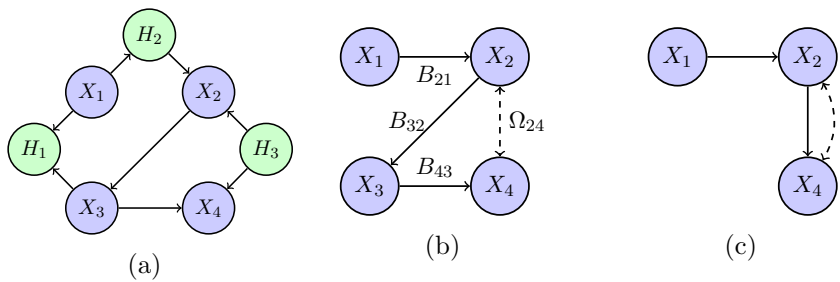
\begin{figure}
  \centering
  \begin{subfigure}{0.3\textwidth}
    \centering
    \resizebox{!}{3cm}{
    \begin{tikzpicture}[->,node distance=2.1cm,thick,
                        main node/.style={circle,fill=blue!20,draw},
                        hidden node/.style={circle,fill=green!20,draw}]
      \node[main node] (1) {$X_1$};
      \node[main node] (2) [right of=1] {$X_2$};
      \node[main node] (3) [below of=1] {$X_3$};
      \node[main node] (4) [right of=3] {$X_4$};
      \node[hidden node] (5) [below right=0.5 of 2] {$H_3$};
      \node[hidden node] (6) [above right=0.5 of 1] {$H_2$};
      \node[hidden node] (7) [below left=0.5 of 1] {$H_1$};

      \path[every node/.style={font=\sffamily}]
        (1) edge node {} (7)
        (1) edge node {} (6)
        (2) edge node {} (3)
        (3) edge node {} (4)
        (3) edge node {} (7)
        (5) edge node {} (2)
        (5) edge node {} (4)
        (6) edge node {} (2);
    \end{tikzpicture}}
    \caption{}
    \label{fig:motivatingexample.dag}
  \end{subfigure}
  \begin{subfigure}{0.3\textwidth}
    \centering
    \resizebox{!}{2.5cm}{
    \begin{tikzpicture}[->,node distance=2.1cm,thick,
                        main node/.style={circle,fill=blue!20,draw}]
      \node[main node] (1) {$X_1$};
      \node[main node] (2) [right of=1] {$X_2$};
      \node[main node] (3) [below of=1] {$X_3$};
      \node[main node] (4) [right of=3] {$X_4$};

      \path[every node/.style={font=\sffamily}]
        (1) edge node [below] {$B_{21}$} (2)
        (2) edge node [left] {$B_{32}$} (3)
            edge [dashed,<->] node [right] {$\Omega_{24}$} (4)
        (3) edge node [above] {$B_{43}$} (4);
    \end{tikzpicture}}
    \caption{}
    \label{fig:motivatingexample.bap}
  \end{subfigure}
  \begin{subfigure}{0.3\textwidth}
    \centering
    \resizebox{!}{2.5cm}{
    \begin{tikzpicture}[->,node distance=2.1cm,thick,
                        main node/.style={circle,fill=blue!20,draw}]
      \node[main node] (1) {$X_1$};
      \node[main node] (2) [right of=1] {$X_2$};
      \node[main node] (4) [right of=3] {$X_4$};

      \path[every node/.style={font=\sffamily}]
        (1) edge node {} (2)
        (2)    edge node {} (4)
            edge [dashed,bend left,<->] node {} (4);
    \end{tikzpicture}}
    \caption{}
    \label{fig:motivatingexample.mag}
  \end{subfigure}
  \caption{(a) DAG with hidden variables $H_1, H_2, H_3$, (b) resulting BAP over the observed variables $X_1, \ldots, X_4$ with annotated edge weights, and (c) resulting graph if $X_3$ is also not observed, which is not a BAP.}
  \label{fig:motivatingexample}
\end{figure}

\subsection*{Challenges}
The main challenge, like with all structure search problems in graphical modelling, is the vastness of the model space. The number of BAPs grows super-exponentially. Hence, as is the case for DAGs, exhaustively scoring all BAPs and finding the global score optimum is very challenging. For DAGs, \citet{SilTM06} proposed a surprisingly simple algorithm whose runtime is exponential in the number of nodes and which is feasible for problems with up to about 30 nodes. However, extending their idea to BAPs is not straightforward, and we aim to deal with settings where the number of nodes can be significantly larger.

Another major challenge, specifically for our setting, is the fact that a graphical characterization of the (distributional) equivalence classes for BAP models is not yet known. In the DAG case, for example, it is known that models are equivalent if and only if they share the same skeleton and v-structures \citep{ VerTP91}. A similar result is not known for BAPs (or the more general acyclic directed mixed graphs). This makes it hard to traverse the search space efficiently, since one cannot search over the equivalence classes (like the greedy equivalence search for DAGs, see \citet{ ChiDB02}). It also makes it difficult to evaluate simulation results, since the graphs corresponding to the ground truth and the optimal solution may be distinct and yet still represent the same model.

\subsection*{Contributions}
We provide the first structure learning algorithm for BAPs. It is a score-based algorithm and uses greedy hill climbing to optimize a penalized likelihood score. We are able to achieve a significant computational speedup by decomposing the score over the bidirected connected components of the graph and caching the score of each component. To mitigate the problem of local optima, we perform many random restarts of the greedy search.

We propose to approximate the distributional equivalence class of a BAP by using a greedy strategy for likelihood scoring. If two BAPs are similar with respect to their penalized likelihoods within a tolerance, they should be treated as statistically indistinguishable and hence as belonging to the same class of (nearly) equivalent BAPs. Based on such greedily computed (near) equivalence classes, we can then infer bounds of total causal effects, in the spirit of \citet{ MaaMA09, MaaMA10}.

We present some theoretical results towards equivalence properties in BAP models, some of which generalize to acyclic path diagrams. In particular, we prove some necessary and some sufficient conditions for BAP equivalence. Furthermore, we present a Markov Chain Monte Carlo method for uniformly sampling BAPs based on ideas from \citet{ KuiJM15}.

We obtain promising results on simulated data sets despite the challenges listed above. Comparing the highest-scoring BAPs and DAGs on real datasets exemplifies the more conservative nature of BAP models.

\subsection*{Related Work}
There are two main research communities that intersect at this topic. On the one side there are the path diagram models, going back to \citet{ WriS34} and then being mainly developed in the behavioral sciences \citep{ JorK70, DunO75, GlyCS86, JorK01}. In this setting a model for the edge functions is assumed, usually a parametric model with linear edge functions and Gaussian noise. In a very general formulation, the graph over the observed variables is assumed to be an acyclic directed mixed graph (ADMG), which can have bows. While in general the parameters for these models are not identified, \citet{ DrtMA11} give necessary and sufficient conditions for global identifiability. Complete necessary and sufficient conditions for the more useful almost everwhere identifiability remain unknown (however, see \citet{FoyRA12} for some necessary and some sufficient conditions). BAP models are a useful subclass, since they are almost everywhere identified \citep{ BriCP02}. \citet{ DrtMA09} provided an algorithm, called residual iterative conditional fitting (RICF), for maximum likelihood estimation of the parameters for a given BAP.

On the other side there are the non-parametric hidden variable models, which are defined as marginalized DAG models \citep{ PeaJ00}\footnote{Strictly speaking, not all SEMs with correlated Gaussian errors can be interpreted as latent variable models, since the latent variable models have additional inequality constraints.  We do not discuss this further here, but see \citet{FoxCA15} for more details.}. The marginalized distributions are constrained by conditional independencies, as well as additional equality and inequality constraints \citep{ EvaR16}. When just modelling the conditional independence constraints, the class of maximal ancestral graphs (MAGs) is sufficient \citep{ RicTS02}. \citet{ ShpI14} have proposed the nested Markov model using ADMGs to also include the additional equality constraints. Finally, mDAGs were introduced to model all resulting constraints \citep{ EvaR16}. In general BAPs induce independence constraints and also Verma constraints \citep[Sections 7.3 and 8]{RicTS02}, as well as other restrictions that do not apply in the non-parametric case.  The BAP in Figure \ref{fig:motivatingexample.bap}, for example, implies a Verma constraint.  Gaussian BAPs are also `maximal', in the sense that every missing edge induces a constraint.  In the non-parametric case, with each additional layer of constraints learning the graphical structure from data becomes more complicated, but at the same time more available information is utilized and a possibly more detailed structure can be learned.  In the Gaussian case, however, all models are parameteric, and fitting BAPs that do not correspond to conditional independence models is essentially no different to fitting those that do.  At the graphical level the search is perhaps easier, since we do not have to place the restriction of ancestrality on the structure of the graph.  However, unlike for MAGs, the equivalence class of BAPs is not known, which means that one may end up fitting the same model multiple times in the form of different graphs. Furthermore, BAPs are easier to interpret as hidden variable models. This can be seen when comparing the BAP in Figure~\ref{fig:motivatingexample.bap} with the corresponding MAG. The latter would have an additional edge between $X_1$ and $X_4$ since there is no (conditional) independency of these two variables. As can be verified, the BAP and the MAG in this example are not distributionally equivalent, since the former encodes additional non-independence constraints.

Structure search for MAGs can be done with the FCI \citep{ SpiPA93}, RFCI \citep{ MaaMA09}, or FCI+ \citep{ ClaTA13} algorithms. \citet{ SilRG06} propose a fully Bayesian method for structure search in linear Gaussian ADMGs, sampling from the posterior distribution using an MCMC approach. \citet{ ShpIA12} employ a greedy approach to optimize a penalized likelihood over ADMGs for discrete parametrizations.

\subsection*{Outline of this Paper}
In Section~\ref{sec:model} we give an in-depth overview of the model and its estimation from data, as well as some distributional equivalence properties. In Section~\ref{sec:greedy} we present the details of our greedy algorithm with various computational speedups. In Section~\ref{sec:empirical} we present empirical results on simulated and real datasets. All proofs as well as further theoretical results and justifications can be found in the Appendix.

\section{Model and Estimation} \label{sec:model}
\subsection{Graph Terminology}
Let $X_1, \ldots, X_d$ be a set of random variables and $V=\{1, \ldots, d\}$ be their index set. The elements of $V$ are also called \emph{nodes} or \emph{vertices}. A \emph{mixed graph} or \emph{path diagram} $G$ on $V$ is an ordered tuple $G=(V, E_D, E_B)$ for some $E_D, E_B \subseteq V \times V \setminus \{ (i,i) ~|~ i \in V \}$. If $(i,j) \in E_D$, we say there is a \emph{directed edge} from $i$ to $j$ and write $i \rightarrow j \in G$. If $(i,j) \in E_B$, we must also have $(j,i) \in E_B$, and we say there is a \emph{bidirected edge} between $i$ and $j$ and write $i \leftrightarrow j \in G$. The set $\mathrm{pa}_G(i) \coloneqq \{ j ~|~ j \rightarrow i \in G \}$ is called the \emph{parents} of $i$. This definition extends to sets of nodes $S$ in the obvious way: $\mathrm{pa}_G(S) \coloneqq \bigcup_{i \in S} \mathrm{pa}_G(i)$. The \emph{in-degree} of $i$ is the number of arrowheads at $i$. If $V' \subseteq V$, $E'_D \subseteq E_D|_{V' \times V'}$, and $E'_B \subseteq E_B|_{V' \times V'}$, then $G' = (V', E'_D, E'_B)$ is called a \emph{subgraph} of $G$, and we write $G' \subseteq G$. The \emph{induced subgraph} $G_W$ for some vertex set $W \subset V$ is the restriction of $V$ to vertices $W$. If $G' \subseteq G$ but $G' \neq G$, we call $G'$ a \emph{strict subgraph} of $G$ and write $G' \subset G$. The \emph{skeleton} of $G$ is the undirected graph over the same node set $V$ and with edges $i-j$ if and only if $i \rightarrow j \in G$ or $i \leftrightarrow j \in G$ (or both).

A \emph{path} $\pi$ between $i$ and $j$ is an ordered tuple of (not necessarily distinct) nodes $\pi = (v_0 = i, \ldots, v_l = j)$ such that there is an edge between $v_k$ and $v_{k+1}$ for all $k=0, \ldots, l-1$. If the nodes are distinct, the path is called \emph{non-overlapping} (note that in the literature a path is mostly defined as non-overlapping). The \emph{length} of $\pi$ is the number of edges $\lambda(\pi) = l$. If $\pi$ consists only of directed edges pointing in the direction of $j$, it is called a \emph{directed path} from $i$ to $j$. A node $j$ on a non-overlapping path $\pi$ is called a \emph{collider} if $\pi$ contains a non-overlapping subpath $(i,j,k)$ with two arrowheads into $j$\footnote{That is, one of the following structures: $\rightarrow \leftarrow, \leftrightarrow \leftarrow, \rightarrow \leftrightarrow, \leftrightarrow \leftrightarrow$.}. Otherwise $j$ is called a non-collider on the path. If $j$ is a collider on a non-overlapping path $(i,j,k)$, we call $(i,j,k)$ a \emph{collider triple}. Moreover, if $(i,j,k)$ is a collider triple and $i$ and $k$ are not adjacent in the graph, then $(i,j,k)$ is called a \emph{v-structure}. A path without colliders is called a \emph{trek}.

Let $A,B,C \subset V$ be three disjoint sets of nodes. The set $\mathrm{an}(C) \coloneqq C \cup \{ i \in V ~|~ \text{there exists a directed path from } i \text{ to } c \text{ for some } c \in C\}$ is called the \emph{ancestors} of $C$. A non-overlapping path $\pi$ from $a \in A$ to $b \in B$ is said to be \emph{m-connecting given $C$} if every non-collider on $\pi$ is not in $C$ and every collider on $\pi$ is in $\mathrm{an}(C)$. If there are no such paths, $A$ and $B$ are \emph{m-separated given $C$}, and we write $A \ci_m B ~|~ C$. We use a similar notation for denoting conditional independence of the corresponding set of variables $X_A$ and $X_B$ given $X_C$: $X_A \ci X_B ~|~ X_C$.

A graph $G$ is called \emph{cyclic} if there are at least two distinct nodes $i$ and $j$ such that there are directed paths both from $i$ to $j$ and from $j$ to $i$. Otherwise $G$ is called \emph{acyclic} or \emph{recursive}. An acyclic path diagram is also called an acyclic directed mixed graph (ADMG). An acyclic path diagram having at most one edge between each pair of nodes is called a \emph{bow-free\footnote{The structure $i \rightarrow j$ together with $i \leftrightarrow j$ is also known as \emph{bow}.} acyclic path diagram (BAP)}. An ADMG without any bidirected edges is called a \emph{directed acyclic graph (DAG)}.

\subsection{The Model}
A \emph{linear structural equation model (SEM)} $M$ is a set of linear equations involving the variables $\mathbf{X} = (X_1, \ldots, X_d)^T$ and some error terms $\bm{\epsilon} = (\epsilon_1, \ldots, \epsilon_d)^T$:
\begin{align} \label{eq:sem.vec}
  \mathbf{X} = B \mathbf{X} + \bm\epsilon,
\end{align}
where $B$ is a real matrix, $\mathrm{cov}(\bm\epsilon) = \Omega$ is a positive semi-definite matrix, and we assume that all variables $\mathbf{X}$ have been normalized to mean zero. $M$ has an associated graph $G$ that reflects the structure of $B$ and $\Omega$. For every non-zero entry $B_{ij}$ there is a directed edge from $j$ to $i$, and for every non-zero entry $\Omega_{ij}$ there is a bidirected edge between $i$ and $j$. Thus we can also write~\eqref{eq:sem.vec} as:
\begin{align} \label{eq:sem}
  X_i = \sum_{j \in \mathrm{pa}_G(i)} B_{ij} X_j + \epsilon_i, \qquad \text{for all } i \in V,
\end{align}
with $\cov (\epsilon_i, \epsilon_j) = \Omega_{ij}$ for all $i,j \in V$.

Our model is a special type of SEM, which we refer to as \emph{Gaussian BAP model}\footnote{All BAP models in this paper are assumed to have a Gaussian parametrization unless otherwise stated.}. In particular, we make the following assumptions:
\begin{enumerate}
  \renewcommand{\theenumi}{(A\arabic{enumi})}
  \renewcommand{\labelenumi}{\theenumi}
  \item \label{ass:normal}
  The errors $\bm\epsilon$ follow a multivariate Normal distribution $\mathcal{N} (\mathbf{0}, \Omega)$.
  \item \label{ass:bap}
  The associated graph $G$ is a BAP.
\end{enumerate}
Assumption~\ref{ass:normal} is not strictly needed for our equivalence results, but we rely on it for fitting the models in practice using the RICF method of \citet{ DrtMA09}.

Often $M$ is specified via its graph $G$, and we are interested to find parameters $\theta_G$ compatible with $G$. We thus define the parameter spaces for the edge weight matrices $B$ (directed edges) and $\Omega$ (bidirected edges) for a given BAP $G$ as
\begin{align*}
  \mathcal{B}_G = \{ B \in \mathbb{R}^{d \times d} ~|~& B_{ij} = 0~\mathrm{if}~j \rightarrow i~\mathrm{is~not~an~edge~in}~G \} \\
  \mathcal{O}_G = \{ \Omega \in \mathbb{R}^{d \times d} ~|~& \Omega_{ij} = 0~\mathrm{if}~ i \ne j \mathrm{~and~} i \leftrightarrow j~\mathrm{is~not~an~edge~in}~G; \\
  & \Omega \text{~is~symmetric~positive~semi-definite} \}
\end{align*}
and the combined parameter space as
\begin{align*}
  \Theta_G = \mathcal{B}_G \times \mathcal{O}_G.
\end{align*}
The covariance matrix for $\mathbf{X}$ is given by:
\begin{align} \label{eq:covmat}
  \phi (\theta) = (I-B)^{-1} \Omega (I-B)^{-T},
\end{align}
where $\phi:\Theta_G \rightarrow \mathcal{S}_G$ maps parameters to covariance matrices, and $\mathcal{S}_G \coloneqq \phi(\Theta_G)$ is the set of covariance matrices compatible with $G$.
Note that $\phi (\theta)$ exists since $G$ is acyclic by~\ref{ass:bap} and therefore $I-B$ is invertible.

We assume that the variables are normalized to have variance $1$, that is, we are interested in the subset $\bar{\mathcal{S}}_G \subset \mathcal{S}_G$, where $\bar{\mathcal{S}}_G = \{ \Sigma \in \mathcal{S}_G ~|~ \Sigma_{ii} = 1 \text{ for all } i = 1, \ldots, d \}$, and its preimage under $\phi$, $\bar{\Theta}_G \coloneqq \phi^{-1} \left( \bar{\mathcal{S}}_G \right) \subset \Theta_G$.

One of the main motivations of working with BAP models is parameter identifiability. This is defined below:

\begin{mydef} \label{identifiable}
  A normalized parameter $\theta_G \in \bar{\Theta}_G$ is identifiable if there is no $\theta_G' \in \bar{\Theta}_G$ such that $\theta_G \ne \theta_G'$ and $\phi(\theta_G) = \phi(\theta_G')$.
\end{mydef}

\citet{ BriCP02} show that for any BAP $G$, the set of normalized non-identifiable parameters has measure zero.

The causal interpretation of BAPs is the following. A directed edge from $X$ to $Y$ represents a direct causal effect of $X$ on $Y$. A bidirected edge between $X$ and $Y$ represents a hidden variable which is a cause of both $X$ and $Y$, see also Figure~\ref{fig:motivatingexample}. In practice, one is often interested in predicting the effect of an intervention at $X_j$ on another variable $X_i$. This is called the \emph{total causal effect} of $X_j$ on $X_i$ and can be defined as $E_{ij} = \frac{\partial}{\partial x} \mathbb{E}[X_i ~|~ do(X_j=x)]$, where the $do(X_j=x)$ means replacing the respective equation in~\eqref{eq:sem} with $X_j=x$ \citep{ PeaJ00}. For linear Gaussian path diagrams such as in \eqref{eq:sem.vec} or \eqref{eq:sem}, this is a constant quantity given by
\begin{align} \label{eq:causaleffect}
  E_{ij} = \left( (I-B)^{-1} \right)_{ij}.
\end{align}

\subsection{Penalized Maximum Likelihood}
Consider a BAP $G$. A first objective is to estimate the parameters $\theta_G$ from $n$ i.i.d.\ samples of model~\eqref{eq:sem}, denoted by $\{ x_i^{(s)} \}$ ($i=1,\ldots ,d$ and $s=1,\ldots ,n$). This can be done by maximum likelihood estimation using the RICF method of \citet{ DrtMA09}. Given the Gaussianity assumption~\ref{ass:normal} and the covariance formula~\eqref{eq:covmat}, one can express the log-likelihood for some given parameters $\theta_G$ and the sample covariance matrix $S$ as:
\begin{align} \label{eq:likelihood}
  l(\theta_G; S) = -\frac{n}{2} \left( \log |2 \pi \Sigma_G| + \frac{n-1}{n} \tr (\Sigma_G^{-1} S) \right),
\end{align}
where $\Sigma_G = \phi(\theta_G)$ is the covariance matrix implied by parameters $\theta_G$, see for example \citet[(4.1.9)]{ MarKA79}. However, due to the structural constraints on $B$ and $\Omega$ it is not straightforward to maximize this for $\theta_G$. RICF is an iterative method to do so, yielding the maximum likelihood estimate:
\begin{align} \label{eq:maxlikelihood}
  \hat{\theta}_G = \argmax_{\theta_G \in \Theta_G} l(S; \theta_G).
\end{align}

We now extend this to the scenario where the graph $G$ is also unknown, using a regularized likelihood score with a BIC-like penalty term that increases with the number of edges. Concretely, we use the following score for a given BAP $G$:
\begin{align} \label{eq:score}
  s(G) \coloneqq \frac{1}{n} \left( \max_{\theta_G \in \Theta_G} l(S; \theta_G) - \left( \# \{ \mathrm{nodes} \} + \# \{ \mathrm{edges} \} \right) \log n \right).
\end{align}
We have scaled the log-likelihood and penalty with $1/n$ so that the score is expected to be $O(1)$ as $n$ increases. Compared with the usual BIC penalty, we chose our penalty to be twice as large, since this led to better performance in simulations studies\footnote{In practice, one could also treat the penalty coefficient as a hyperparameter and choose it via cross-validation.}. The number of nodes is typically fixed, so does not matter for comparing graphs over the same vertex set. We included it to make explicit the penalization of the model parameters (which correspond to nodes and edges).

In our search for the true causal graph $G$, we assume that if $\Sigma \in \bar{\mathcal{S}}_{H}$ for any other graph $H$, then $\bar{\mathcal{S}}_{H} \supseteq \bar{\mathcal{S}}_{G}$, that is $H$ represents a strict supermodel of $G$.  This rules out the possibility that `by chance' we land on a distribution contained in a submodel, and is a minimal requirement for causal learning. The set of matrices that violate the requirement has measure zero within the model $\bar{\mathcal{S}}_{G}$ (assuming entries in $B$ and $\Omega$ are generated according to a positive joint density with respect to the Lebesgue measure)\footnote{This follows because the models are parametrically defined algebraic varieties, which are therefore irreducible.  Any sub-variety of $\bar{\mathcal{S}}_G$, such as that achieved by intersecting with another model is either equal to $\bar{\mathcal{S}}_G$ or has strictly smaller dimension.  See, for example, \citet{CoxDA07}.}.  This requirement is analogous to the faithfulness assumption of \citet{SpiPA93}, though faithfulness applies separately to individual conditional independence constraints rather than to the entire model.

\subsection{Equivalence Properties}
There is an important issue when doing structure learning with graphical models: typically the maximizers of~\eqref{eq:score} will not be unique. This is a fundamental problem for most model classes and a consequence of the model being underdetermined. In general, there are sets of graphs that are statistically indistinguishable (in the sense that they can all parametrize the same joint distributions over the variables). These graphs are called \emph{distributionally equivalent}. For nonparametric DAG models (without non-linearity or non-Gaussianity constraints), for example, the distributional equivalence classes are characterized by conditional independencies and are called Markov equivalence classes. For BAPs, distributional equivalence is not completely characterized yet (see \citet{ SpiPA98} or \citet{WilL12} for a discussion of the linear Gaussian ADMG case), but we present some necessary and some sufficient conditions, that can be used to simplify structure search in practice. Let us first make precise the different notions of model equivalence.
\begin{mydef}
  Two BAPs $G_1, G_2$ over a set of nodes $V$ are Markov equivalent if they imply the same m-separation relationships.
\end{mydef}
This essentially means they imply the same conditional independencies, and the definition coincides with the classical notion of Markov equivalence when $G_1$ and $G_2$ are both DAGs. The following notion of distributional equivalence is stronger.
\begin{mydef}
  Two BAPs $G_1, G_2$ are distributionally equivalent if $\bar{\mathcal{S}}_{G_1} = \bar{\mathcal{S}}_{G_2}$.
\end{mydef}

We now present some sufficient and some necessary conditions for distributional equivalence in BAP models. Note that the Gaussianity assumption~\ref{ass:normal} is not required for these to hold.

\subsubsection{Necessary Conditions}
\citet{ SpiPA98} showed the following global Markov property for general linear path diagrams: if there are nodes $a,b \in V$ and a possibly empty set $C \subset V$ such that $a \ci_m b ~|~ C$, then the partial correlation of $X_a$ and $X_b$ given $X_C$ is zero. In addition, if such an m-separation does not hold then the partial correlation is non-zero for almost all distributions. As a direct consequence, we get the following first result:

\begin{lemma} \label{lemma:m-sep}
  If two BAPs $G_1, G_2$ do not share the same m-separations, they are not distributionally equivalent.
\end{lemma}

Unlike for DAGs, the converse is not true, as the counterexample in Figure~\ref{fig:counterexample} shows. For DAGs it is trivial to show that having the same skeleton is necessary for Markov equivalence, since a missing edge between two nodes means they can be d-separated, and thus a conditional independency would have to be present in the corresponding distribution. For BAPs a missing edge does not necessarily result in an m-separation, as the counterexample in Figure~\ref{fig:counterexample} shows. However, the following result will allow us to improve the necessary condition of same m-separations for BAPs to the same as for DAGs.

\begin{theorem} \label{thm:subgraph}
Let $G$ and $G'$ be distributionally equivalent BAPs on vertices $V$. Then, for any subset $W \subseteq V$, the induced subgraphs $G_W$ and $G'_W$ are also distributionally equivalent.
\end{theorem}

If we in particular look at the induced subgraphs of size two and three we obtain the following necessary conditions for distributional equivalence.

\begin{corollary} \label{cor:necessary}
Let $G$ and $G'$ be distributionally equivalent BAPs. Then they have the same skeleton and v-structures.
\end{corollary}

Since m-separations are not fully determined by the skeleton and the v-structures of a graph, it is also worthwhile to look at larger subgraphs. This leads, for example, to the following result: if two graphs are distributionally equivalent and a particular path is a so-called discriminating path in both graphs, then the discriminated triple will be a collider in both or in neither \citep[see][Section 3.4]{AliA09}.  

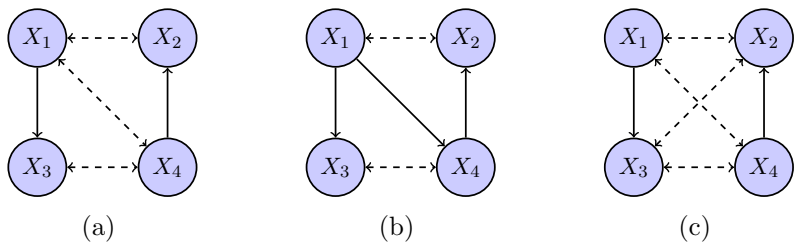
\begin{figure}
  \centering
  \begin{subfigure}{0.3\textwidth}
    \centering
    \resizebox{!}{2.5cm}{
    \begin{tikzpicture}[->,node distance=2.1cm,thick,
                        main node/.style={circle,fill=blue!20,draw}]
      \node[main node] (1) {$X_1$};
      \node[main node] (2) [right of=1] {$X_2$};
      \node[main node] (3) [below of=1] {$X_3$};
      \node[main node] (4) [below of=2] {$X_4$};

      \path[every node/.style={font=\sffamily}]
        (1) edge [dashed,<->] node {} (2)
            edge node {} (3)
            edge [dashed,<->] node {} (4)
        (3) edge [dashed,<->] node {} (4)
        (4) edge node {} (2);
    \end{tikzpicture}}
    \caption{}
    \label{fig:example.nonmaximal}
  \end{subfigure}
  \begin{subfigure}{0.3\textwidth}
    \centering
    \resizebox{!}{2.5cm}{
    \begin{tikzpicture}[->,node distance=2.1cm,thick,
                        main node/.style={circle,fill=blue!20,draw}]
      \node[main node] (1) {$X_1$};
      \node[main node] (2) [right of=1] {$X_2$};
      \node[main node] (3) [below of=1] {$X_3$};
      \node[main node] (4) [below of=2] {$X_4$};

      \path[every node/.style={font=\sffamily}]
        (1) edge [dashed,<->] node {} (2)
            edge node {} (3)
            edge node {} (4)
        (3) edge [dashed,<->] node {} (4)
        (4) edge node {} (2);
    \end{tikzpicture}}
    \caption{}
    \label{fig:example.msep}
  \end{subfigure}
  \begin{subfigure}{0.3\textwidth}
    \centering
    \resizebox{!}{2.5cm}{
    \begin{tikzpicture}[->,node distance=2.1cm,thick,
                        main node/.style={circle,fill=blue!20,draw}]
      \node[main node] (1) {$X_1$};
      \node[main node] (2) [right of=1] {$X_2$};
      \node[main node] (3) [below of=1] {$X_3$};
      \node[main node] (4) [below of=2] {$X_4$};

      \path[every node/.style={font=\sffamily}]
        (1) edge [dashed,<->] node {} (2)
            edge node {} (3)
            edge [dashed,<->] node {} (4)
        (3) edge [dashed,<->] node {} (4)
            edge [dashed,<->] node {} (2)
        (4) edge node {} (2);
    \end{tikzpicture}}
    \caption{}
    \label{fig:example.maximal}
  \end{subfigure}
  \caption{The two BAPs in (a) and (b) share the same skeleton and v-structures, but in (a) there are no m-separations, whereas in (b) we have $X_2 \ci_m X_3 ~|~ \{X_1, X_4\}$. BAPs (a) and (c) share the same m-separations (none) but are not distributionally equivalent since they have different skeletons (using Corollary~\ref{cor:necessary}).}
  \label{fig:counterexample}
\end{figure}

The criteria given above are not complete, in the sense that there exist BAPs that are not distributionally equivalent and yet this cannot be proven by applying these results.  For example, the BAPs in Figure \ref{fig:counterexample2} are not distributionally equivalent, which can be shown using the results of \citet{ShpI14}.  However, they both have no m-separations. A complete characterization remains an open problem.

\subsubsection{Sufficient Conditions}
To prove sufficient conditions, we first give a characterization of the equivalence class in terms of treks (collider-free paths) using Wright's path tracing formula \citep{ WriS60}. Wright's formula expresses the covariance between any two variables in a path diagram as the sum-product over the edge labels of the treks between those variables, as long as all variables are normalized to variance~$1$. A precise statement as well as a proof of a more general version of Wright's formula can be found in the Appendix (Theorems \ref{thm:unstandardised} and \ref{thm:wright}). As an example, consider the BAP in Figure~\ref{fig:motivatingexample.bap}. There are two treks between $X_2$ and $X_4$: $X_2 \rightarrow X_3 \rightarrow X_4$ and $X_2 \leftrightarrow X_4$. Hence $\cov (X_2, X_4) = B_{32} B_{43} + \Omega_{24}$, assuming normalized parameters. Similarly we have $\cov (X_1, X_4) = B_{21} B_{32} B_{43}$.

As a consequence of Wright's formula, we can show that having the same skeleton and collider triples is sufficient for two acyclic path diagrams to be distributionally equivalent:

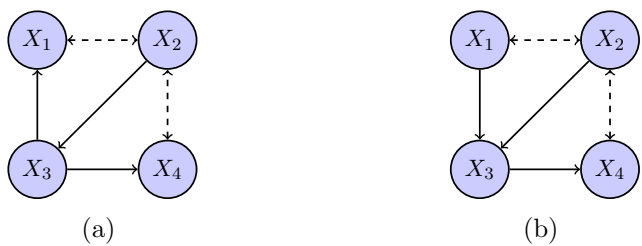
\begin{figure}
  \centering
  \begin{subfigure}{0.45\textwidth}
    \centering
    \resizebox{!}{2.5cm}{
    \begin{tikzpicture}[->,node distance=2.1cm,thick,
                        main node/.style={circle,fill=blue!20,draw}]
      \node[main node] (1) {$X_1$};
      \node[main node] (2) [right of=1] {$X_2$};
      \node[main node] (3) [below of=1] {$X_3$};
      \node[main node] (4) [below of=2] {$X_4$};

      \path[every node/.style={font=\sffamily}]
        (1) edge [dashed,<->] node {} (2)
          (2)  edge node {} (3)
        (3) edge node {} (4)
        (3) edge node {} (1)
        (4) edge [dashed,<->] node {} (2);
    \end{tikzpicture}}
    \caption{}
    \label{fig:example.nonmaximal}
  \end{subfigure}
  \begin{subfigure}{0.45\textwidth}
    \centering
    \resizebox{!}{2.5cm}{
  \begin{tikzpicture}[->,node distance=2.1cm,thick,
                        main node/.style={circle,fill=blue!20,draw}]
      \node[main node] (1) {$X_1$};
      \node[main node] (2) [right of=1] {$X_2$};
      \node[main node] (3) [below of=1] {$X_3$};
      \node[main node] (4) [below of=2] {$X_4$};

      \path[every node/.style={font=\sffamily}]
       (1) edge [dashed,<->] node {} (2)
          (2)  edge node {} (3)
        (3) edge node {} (4)
        (1) edge node {} (3)
        (4) edge [dashed,<->] node {} (2);
         \end{tikzpicture}}
    \caption{}
    \label{fig:example.msep}
  \end{subfigure}
  \caption{The two BAPs in (a) and (b) differ only in the direction of the $X_1$, $X_3$ edge; both have no m-separations, 
  and every induced subgraph leads to models which are distributionally equivalent.  However, by using the results of 
  \citet{ShpI14} one can show that these models are not distributionally equivalent.}
  \label{fig:counterexample2}
\end{figure}

\begin{theorem} \label{thm:sufficient}
  Let $G_1, G_2$ be two acyclic path diagrams that have the same skeleton and collider triples. Then $G_1$ and $G_2$ are distributionally equivalent.
\end{theorem}

Considering Figure~\ref{fig:example.msep}, for example, this result shows that if we replace the $X_1 \leftrightarrow X_2$ edge with $X_1 \rightarrow X_2$, the resulting graph is distributionally equivalent to the original.

For DAGs, it is known that the weaker condition of having the same skeleton and v-structures is sufficient for being Markov equivalent. For BAPs this is not true, as the counterexample in Figure~\ref{fig:counterexample} (together with Lemma~\ref{lemma:m-sep}) shows.

We therefore have that the distributional equivalence class of a BAP $G$:
\begin{itemize}
\item
is contained in the set of BAPs with the same skeleton and v-structures as $G$ and
\item
contains the set of BAPs with the same skeleton and collider triples as $G$.
\end{itemize}
We know that the first relation is strict by the counterexample mentioned above and have strong evidence that the second relation is strict as well~\citep[Appendix B]{NowC15}\footnote{These empirical results suggest all 3-node full BAPs to be distributionally equivalent, which would mean there are distributionally equivalent BAPs with different collider triples, implying the strictness of the second inclusion relation above.}.

\section{Greedy Search} \label{sec:greedy}
We aim to find the maximizer of~\eqref{eq:score} over all graphs over the node set $V=\{1, \ldots, d\}$. Since exhaustive search is infeasible, we use greedy hill-climbing. Starting from some graph $G^0$, this method obtains increasingly better estimates by exploring the local neighborhood of the current graph. At the end of each exploration, the highest-scoring graph is selected as the next estimate. This approach is also called greedy search and is often used for combinatorial optimization problems. Greedy search converges to a local optimum, although typically not the global one. To alleviate this we repeat it multiple times with different (random) starting points.

We use the following neighborhood relation. A BAP $G'$ is in the \emph{local neighborhood} of $G$ if it differs by exactly one edge, that is, the number of edges differs by at most one, and one of the following holds:
\begin{enumerate}
\item
$G \subset G'$ (edge addition),
\item
$G' \subset G$ (edge deletion), or
\item
$G$ and $G'$ have the same skeleton (edge change).
\end{enumerate}
If we only admit the first condition, the procedure is called \emph{forward search}, and it is usually started with the empty graph. Instead of searching through the complete local neighborhood at each step (which can become prohibitive for large graphs), we can also select a random subset of neighbors and only consider those.

In Sections \ref{sec:score} and \ref{sec:uniform} we describe some adaptations of this general scheme, that are specific to the problem of BAP learning. In Section~\ref{sec:greedy.equivalence} we describe our greedy equivalence class algorithm.

\subsection{Score Decomposition} \label{sec:score}
Greedy search becomes much more efficient when the score separates over the nodes or parts of the nodes. For DAGs, for example, the log-likelihood can be written as a sum of components, each of which only depends on one node and its parents. Hence, when considering a neighbor of some given DAG, one only needs to update the components affected by the respective edge change. A similar property holds for BAPs. Here, however, the components are not the nodes themselves, but rather the connected components of the bidirected part of the graph (that is, the partition of $V$ into sets of vertices that are reachable from each other by only traversing the bidirected edges). For example, in Figure~\ref{fig:motivatingexample.bap} the bidirected connected components (sometimes also called districts) are $\{X_1\}, \{X_2, X_4\}, \{X_3\}$. This decomposition property is known \citep{ TiaJ05, RicT09}, but for completeness we give a derivation in appendix~\ref{sec:likelihoodseparation}. We write out the special case of the Gaussian parametrization below.

Let us write $p_G^{\mathbf{X}}$ for the joint density of $\mathbf{X}$ under the model~\eqref{eq:sem}, and $p_G^{\bm\epsilon}$ for the corresponding joint density of $\bm\epsilon$. Let $C_1, \ldots, C_K$ be the connected components of the bidirected part of $G$. We separate the model $G$ into submodels $G_1, \ldots, G_K$ of the full SEM~\eqref{eq:sem}, where each $G_k$ consists only of nodes in $V_k = C_k \cup \mathrm{pa}(C_k)$ and without any edges between nodes in $\mathrm{pa}(C_k) \setminus C_k$. Then, as we show in appendix~\ref{sec:likelihoodseparation}, the log-likelihood of the model with joint density $p_G^{\mathbf{X}}$ given data $\mathcal{D} = \{x_i^{(s)}\}$ (with $1 \le i \le d$ and $1 \le s \le n$) can be written as:
\begin{align*}
  l (p_G^{\mathbf{X}}; \mathcal{D}) &= \sum_{s=1}^n \log p_G^{\mathbf{X}} (x_1^{(s)}, \ldots, x_p^{(s)}) \notag\\
  &= \sum_k \left( l (p_{G_k}^{\mathbf{X}}; \{x_i^{(s)}\}^{s=1,\ldots,n}_{i \in V_k}) - \sum_{j \in \mathrm{pa} (C_k) \setminus C_k} l (p_{G_k}^{\mathbf{X}}; \{x_j^{(s)}\}^{s=1,\ldots,n}) \right),
\end{align*}
where $l (p_{G_k}^{\mathbf{X}}; \{x_j^{(s)}\}^{s=1,\ldots,n})$ refers to the likelihood of the $X_j$-marginal of $p_{G_k}^{\mathbf{X}}$. For our Gaussian parametrization, using~\eqref{eq:likelihood}, this becomes
\begin{align*}
  l ( \Sigma_{G_1}, \ldots, \Sigma_{G_K} ; S) & = \\
  -\frac{n}{2} \sum_k & \Bigg( |C_k| \log 2 \pi + \log \frac{|\Sigma_{G_k}|}{\prod_{j \in \mathrm{pa}(C_k) \setminus C_k} \sigma_{kj}^2} + \\
  & \qquad \qquad \frac{n-1}{n} \mathrm{tr} \big(\Sigma_{G_k}^{-1} S_{G_k} - |\mathrm{pa}(C_k) \setminus C_k| \big) \Bigg),
\end{align*}
where $S_{G_k}$ is the restriction of $S$ to the rows and columns corresponding to $C_k$, and $\sigma_{kj}^2$ is the diagonal entry of $\Sigma_{G_k}$ corresponding to parent node $j$. Note that now the log-likelihood depends on $\{x_i^{(s)}\}$ and $p_G^{\mathbf{X}}$ only via $S$ and $\Sigma_{G_1}, \ldots, \Sigma_{G_K}$. Furthermore, the log-likelihood is now a sum of contributions from the submodels $G_k$. This means we only need to re-compute the likelihood of the submodels that are affected by an edge change when scoring the local neighborhood. In practice, we also cache the submodel scores, that is, we assign each encountered submodel a unique hash and store the respective scores, so they can be re-used.

\subsection{Uniformly Random Restarts} \label{sec:uniform}
To restart the greedy search we need random starting points (BAPs), and it seems desirable to sample them \emph{uniformly} at random\footnote{Another motivation for uniform BAP generation is generating ground truths for simulations.}. Just like for DAGs, it is not straightforward to achieve this. What is often done in practice is uniform sampling of triangular (adjacency) matrices and subsequent uniform permutation of the nodes. However, this does not result in uniformly distributed graphs, since for some triangular matrices many permutations yield the same graph (the empty graph is an extreme example). The consequence is a shift of weight to more symmetric graphs, that are invariant under some permutations of their adjacency matrices. A simple example with BAPs for $d=3$ is shown in Figure~\ref{fig:uniform}. One way around this is to design a random process with graphs as states and a uniform limiting distribution. A corresponding Markov Chain Monte Carlo (MCMC) approach is described for example in \citet{ MelGA01} for the case of DAGs. See also \citet{ KuiJM15} for an overview of different sampling schemes.

\begin{figure}
  \centering
  \includegraphics[width=0.5\textwidth]{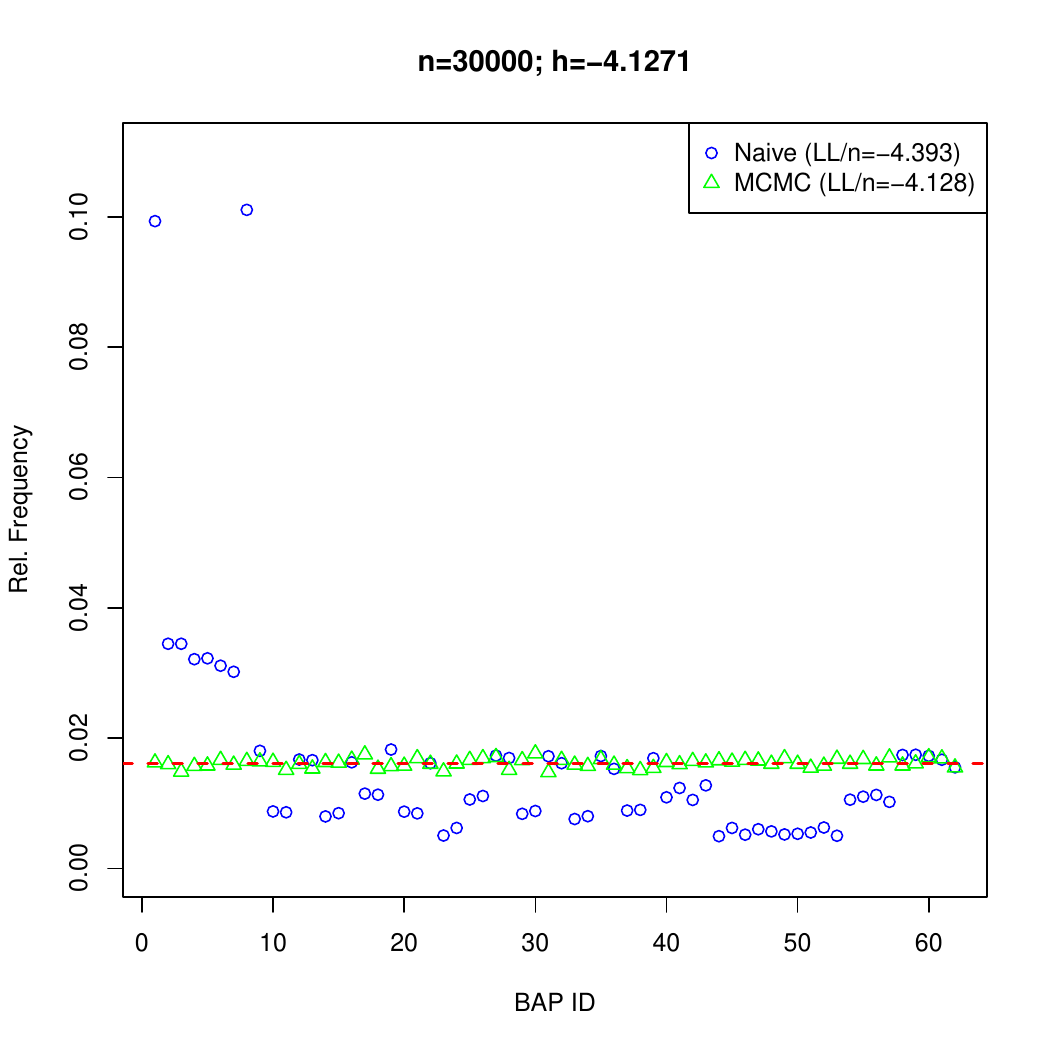}
  \caption{Relative frequencies of the 62 3-node BAPs when sampled 30000 times with the ``naive'' (triangular matrix sampling) and the MCMC method.}
  \label{fig:uniform}
\end{figure}

We adapted the MCMC algorithm for BAPs as described below.

\begin{algorithm}
  \label{alg:mcmc}
  Let $G_k = (V, E_D, E_B)$ be the BAP of the current MCMC iteration. Let $(i,j) \in V \times V \setminus \{ (i,i) ~|~ i \in V \}$ be a position sampled uniformly at random and let $\sigma \in Bernoulli(0.5)$ be a single Bernoulli draw. We then form $G_{k+1} = (V, E_D', E_B')$ by applying the following rules.
  \begin{enumerate}
    \item
    If there is an edge at $(i,j)$ (i.e.\ if $(i,j) \in E_D$ or $(j,i) \in E_D$ or $(i,j) \in E_B$), and
      \begin{enumerate}
        \item
        \label{edge:remove}
        if $\sigma = 0$: remove the edge (i.e.\ $E_D' = E_D \setminus \{ (i,j), (j,i) \}$, $E_B' = E_B \setminus \{ (i,j), (j,i) \}$).
        \item
        \label{edge:nothing}
        if $\sigma = 1$: do nothing.
      \end{enumerate}
    \item
    If there is no edge at $(i,j)$ (i.e.\ if $(i,j) \notin E_D$ and $(j,i) \notin E_D$ and $(i,j) \notin E_B$), and
    \begin{enumerate}
      \item
      \label{noedge:directed}
      if $\sigma = 0$: add $i \rightarrow j$ (i.e.\ $E_D' = E_D \cup \{ (i,j) \}$, $E_B' = E_B$) as long as this does not create a directed cycle, otherwise do nothing;
      \item
      \label{noedge:bidirected}
      if $\sigma = 1$: add $i \leftrightarrow j$ (i.e.\ $E_D' = E_D$, $E_B' = E_B \cup \{ (i,j), (j,i) \}$).
    \end{enumerate}
  \end{enumerate}
\end{algorithm}
It is easy to check that the resulting transition matrix is irreducible and symmetric (see Appendix~\ref{sec:mcmc}) and hence the Markov chain has a (unique) uniform stationary distribution. Thus, starting from any graph, after an initial burn-in period, the distribution of the visited states will be approximately uniform over the set of all BAPs. In practice, we start the process from the empty graph and sample after taking $O(d^4)$ steps (c.f.\ \citet{ KuiJM15}).

It is straightforward to adapt this sampling scheme to a number of constraints, for example uniform sampling over all BAPs with a given maximal in-degree.

\subsection{Greedy Equivalence Class Construction} \label{sec:greedy.equivalence}
We propose the following recursive algorithm to greedily estimate the distributional equivalence class $EC(G)$ of a given BAP $G$ with score $\zeta$. We start by populating the \emph{empirical equivalence class} $\widehat{EC}(G)$ with graphs that have the same skeleton and collider triples as $G$, since these are guaranteed to be equivalent by Theorem~\ref{thm:sufficient}. This is a significant computational shortcut, since these graphs do not have to be found greedily anymore. Then, starting once from each of the graphs in $\widehat{EC}(G)$ found above, at each recursion level we search all edge-change neighbors of the current BAP for BAPs that have a score within $\epsilon$ of $\zeta$ (edge additions or deletions would result in non-equivalent graphs by Corollary~\ref{cor:necessary}). For each such BAP, we start a new recursive search until a maximum depth of $d(d-1)/2$ (corresponding to the maximum number of possible edges) is reached, and always comparing against the original score $\zeta$. Already visited states are stored and ignored. Finally, all found graphs are added to $\widehat{EC}(G)$. The main tuning parameter here is $\epsilon$, essentially specifying the threshold for numerical error, as well as statistically indistinguishable graphs. This results in conservative estimates for total causal effects using the methods discussed in Section~\ref{sec:simulations} by also including neighboring equivalence classes, that are statistically indistiguishable from the given data.

\subsection{Implementation}
Our implementation is done in the statistical computing language R \citep{RCo15}, and the code is available as an R package on github \citep{NowC17}. We make heavy use of the RICF implementation \texttt{fitAncestralGraph}\footnote{Despite the function name the implementation is not restricted to ancestral graphs.} in the \texttt{ggm} package \citep{GioMA15}. We noted that there are sometimes convergence issues, so we adapted the implementation of RICF to include a maximal iteration limit (which we set to 10 by default).

\section{Empirical Results} \label{sec:empirical}
In this section we present some empirical results to show the effectiveness of our method. First, we consider a simulation setting where we can compare against the ground truth. Then we turn to a well known genomic data set, where the ground truth is unknown, but the likelihood of the fitted models can be compared against other methods.

\subsection{Causal Effects Discovery on Simulated Data} \label{sec:simulations}
To validate our method, we randomly generate ground truths, simulate data from them, and try to recover the true total causal effects from the generated datasets. This procedure is repeated $N=100$ times and the results are averaged. We now discuss each step in more detail.
\\ \\
{\bf Randomly generate a BAP $G$.}\\
We do this \emph{uniformly} at random (for a fixed model size $d=10$ and maximal in-degree $\alpha=2$). The sampling procedure is described in Section~\ref{sec:uniform}.
\\ \\
{\bf Randomly generate parameters $\theta_G$.}\\
We sample the directed edge labels in $B$ independently from a standard Normal distribution. We do the same for the bidirected edge labels in $\Omega$, and set the error variances (diagonal entries of $\Omega$) to the respective row-sums of absolute values plus an independently sampled $\chi^2(1)$ value\footnote{By Gershgorin's circle theorem, this is guaranteed to result in a positive definite matrix. To increase stability, we also repeat the sampling of $\Omega$ if its minimal eigenvalue is less then $10^{-6}$.}.
\\ \\
{\bf Simulate data $\{ x_i^{(s)} \}$ from $\theta_G$.}\\
This is straightforward, since we just need to sample from a multivariate Normal distribution with mean $\mathbf{0}$ and covariance $\phi(\theta_G)$. We use the function \texttt{rmvnorm()} from the package \texttt{mvtnorm} \citep{GenAA14}.
\\ \\
{\bf Find an estimate $\hat{G}$ from $\{ x_i^{(s)} \}$.}\\
We use greedy search with $R=100$ uniformly random restarts (as outlined in Section~\ref{sec:greedy}), as well as one greedy forward search starting from the empty model.
\\ \\
{\bf Compare $G$ and $\hat{G}$.}\\
A direct comparison of the graphs does not make sense since they could be different but in the same equivalence class. We therefore estimate the equivalence classes of both $G$ and $\hat{G}$ using the greedy approach described in Section~\ref{sec:greedy.equivalence} with $\epsilon=10^{-10}$ to get $\widehat{EC}(G)$ and $\widehat{EC}(\hat{G})$.

Since the estimated equivalence classes are empirical, it is not straightforward to compare them. For one, they might be intersecting, but not equal (if the recursion level was set too low and they were started from different graphs for example). More relevantly, they might be entirely different, but still agree in large areas of the graph. We therefore chose to evaluate not the graph structure but the identifiability of causal effects. Often this is also more relevant in practice. \citet{ MaaMA09} developed a method (which they called IDA) to find identifiable causal effects in a multiset of DAGs. We apply the same idea in our setting. Specifically, this means we estimate the causal effects matrix $\hat{E}$ for each graph $G' \in \widehat{EC}(G)$ (using the estimated parameters $\hat{\theta}_{G'}=(\hat{B'}, \hat{\Omega'})$ and~\eqref{eq:causaleffect}). We then take absolute values and take the entry-wise minima over all $\hat{E}$ to obtain $\hat{E}^{min}_G$, the minimal absolute causal effects matrix (if an entry $E_{ij}$ is nonzero, there is a nonzero causal effect from $X_i$ to $X_j$ for every graph in the equivalence class). We do the same for $\hat{G}$ to get $\hat{E}^{min}_{\hat{G}}$.

What is left is to compare the minimal absolute causal effects matrix $\hat{E}^{min}_G$ of the ground truth to the minimal absolute causal effects matrix $\hat{E}^{min}_{\hat{G}}$ of the estimate. Thus, our target set consists of all pairs $(i,j)$, such that $(\hat{E}_{G}^{min})_{ij} > 0$. We score the pairs according to our estimated $\hat{E}_{\hat{G}}^{min}$ values, and we report the area under the ROC curve (AUC, see \citet{ HanJA83}). The AUC ranges from 0 to 1, with 1 meaning perfect classification and 0.5 being equivalent to random guessing\footnote{Some care has to be taken because of the fact that the cases $(\hat{E}^{min}_G)_{ij} > 0$ and $(\hat{E}^{min}_G)_{ji} > 0$ exclude each other, but we took this into account when computing the false positive rate.}. In our case, we have a separate ROC curve for each graph. The points on the curve correspond to the thresholding on the estimated absolute value of the causal effects; the $k$-th point shows the situation when we classify the largest $k-1$ values as causal, and the rest as non-causal.
\\ \\
The results for 100 simulations can be seen in Figure~\ref{fig:simulation}; the average AUC is $0.75$. While this suggests that perfect graph discovery is usually not achieved, causal effects can be identified to some extent. We also note that our simulation setting is challenging, in the sense that non-zero edge weights can be arbitrarily close to zero. The computations took $2.5$ hours on an AMD Opteron 6174 processor using $20$~cores.

\begin{figure}
  \centering
  \includegraphics[width=0.5\textwidth]{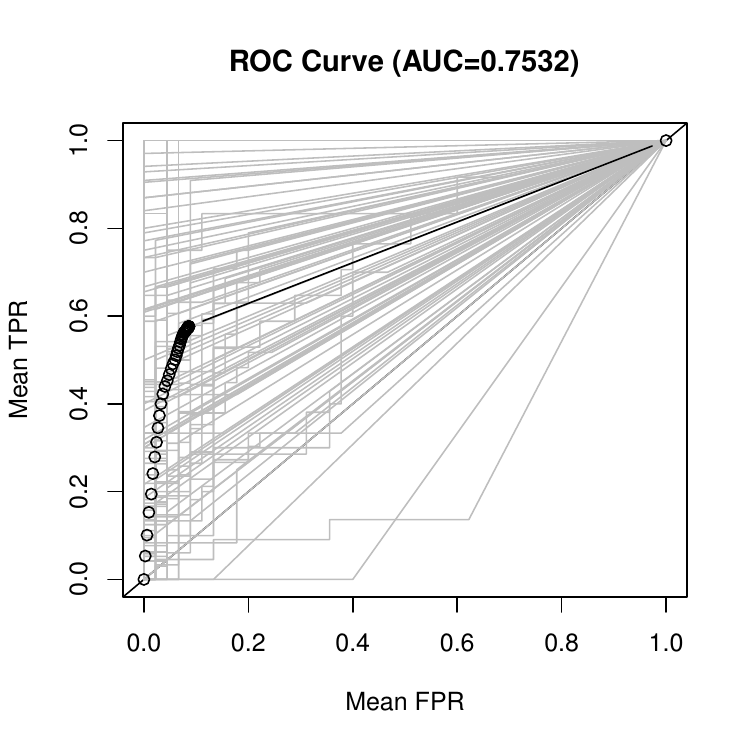}
  \caption{ROC curves for causal effect discovery for $N=100$ simulation runs of BAPs with $d=10$ nodes and a maximal in-degree of $\alpha=2$. Sample size was $n=1000$, greedy search was repeated $R=100$ times at uniformly random starting points. The average area under the ROC curves (AUC) is~$0.75$. The thick curve is the point-wise average of the individual ROC curves.}
  \label{fig:simulation}
\end{figure}

\subsection{Genomic Data}
We also applied our method to a well-known genomics data set \citep{ SacKA05}, where the expression of 11 proteins in human T-cells was measured under 14 different experimental conditions (the sample size varies between 707 and 927). There are likely hidden confounders, which makes this setting suitable for hidden variable models. However, it is questionable whether the bow-freeness, linearity, and Gaussianity assumptions hold to a reasonable approximation (in fact the data seem not to be multivariate normal). Furthermore, there does not exist a ground truth network (although some individual links between pairs of proteins are reported as known in the original paper). So we abstain from comparing a ``best'' network with reported links in literature, but instead use this as an example for comparing highest-scoring BAPs and DAGs.

To do this, we first log-transform all variables since they are heavily skewed. We then run two sets of greedy searches for each of the 14 data sets: one with BAPs and one with DAGs. We use 100 random restarts in both cases. The results can be seen in Figures~\ref{fig:real} and~\ref{fig:real2}. The computations took $4$ hours for the BAP models and $1.5$ hours for the DAG models on an AMD Opteron 6174 processor using $20$ cores.

Note that while the BAPs and DAGs look very similar in many cases, the BAPs are more conservative in identifying causal effects. Eg for dataset~4 there is a v-structure at \texttt{pip3} (with \texttt{pip2} and \text{plcg}) in both the highest-scoring BAP and DAG. However, by Theorem~\ref{thm:sufficient}, this part of the BAP is equivalent to versions with different edge directions (as long as the collider is preserved). This is not the case for the DAG. Hence, in the DAG model these edges are identifiable, but this identifiability disappears in the presence of potential hidden confounders in BAPs. This exemplifies the more conservative nature of BAP models. Another example is the v-structure at \texttt{pakts473} (with \texttt{pka} and \texttt{pkc}) in dataset~8.

\begin{figure}
  \centering
  \includegraphics[width=\textwidth]{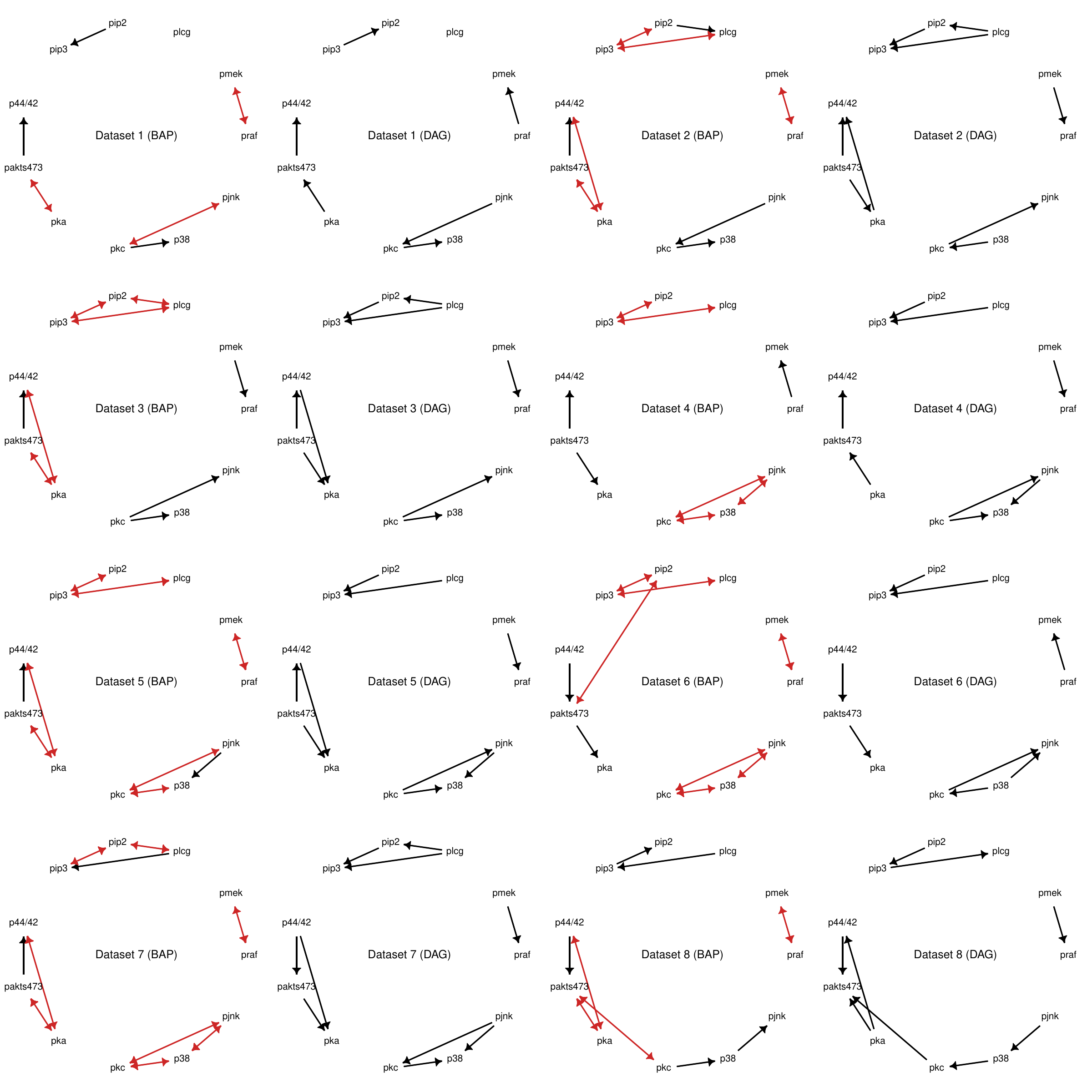}
  \caption{Highest-scoring BAPs and DAGs found by greedy search for 8 of the 14 genomic datasets in~\citet{SacKA05} (continued in Figure~\ref{fig:real2}). For simplicity only one highest-scoring graph is shown per example while equivalent and equally high-scoring graphs are omitted. Note that the equivalence classes in the corresponding BAPs and DAGs are similar but some v-structures lead to identifiablity in DAGs but not in BAPs.}
  \label{fig:real}
\end{figure}

\begin{figure}
  \centering
  \includegraphics[width=\textwidth]{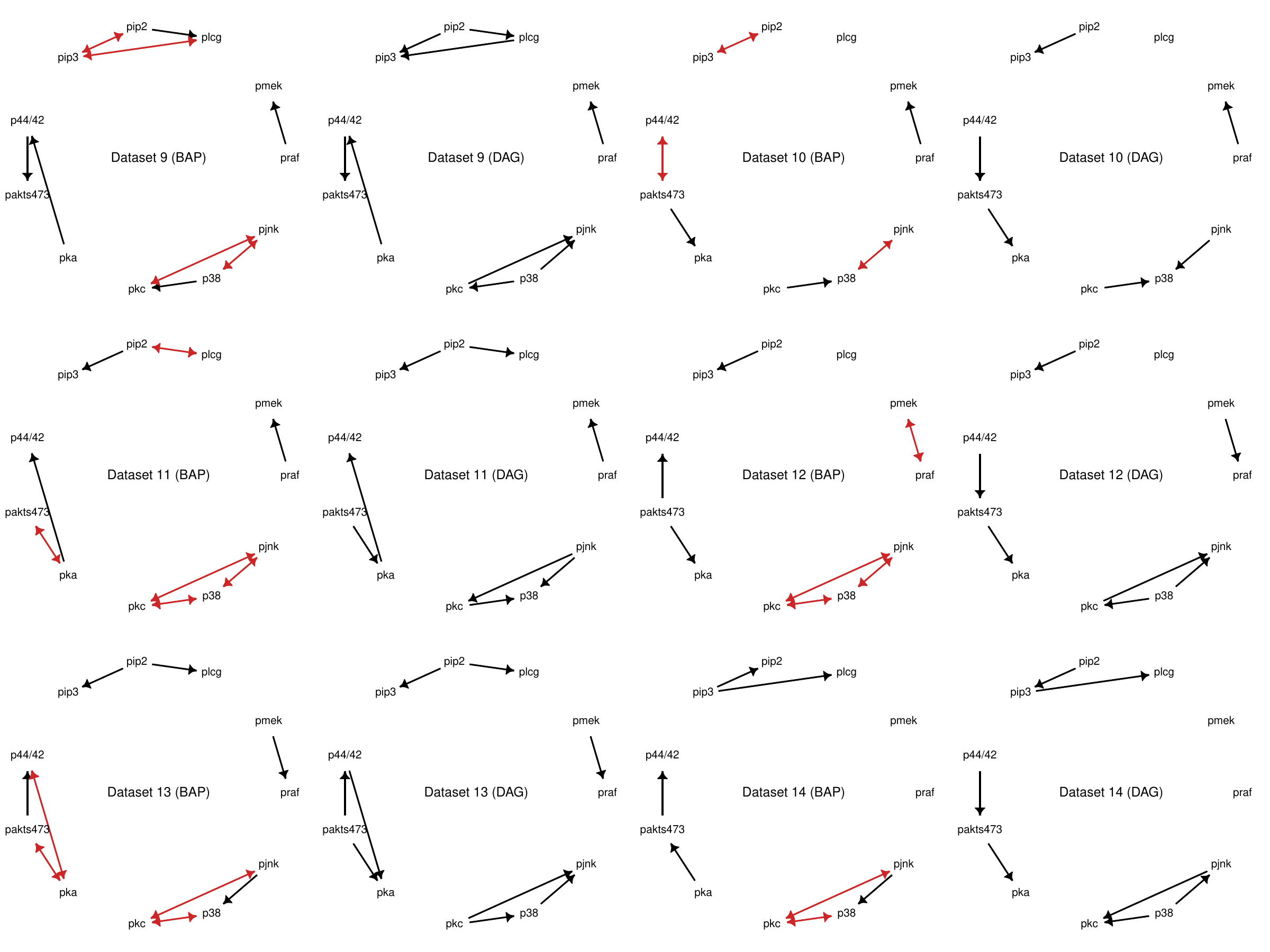}
  \caption{Highest-scoring BAPs and DAGs found by greedy search for the remaining 6 genomic datasets in~\citet{SacKA05} (continuation of Figure~\ref{fig:real}).\\
  Dataset Names:\\
  1: \texttt{cd3cd28}\\2: \texttt{cd3cd28icam2+aktinhib}\\3: \texttt{cd3cd28icam2+g0076}\\4: \texttt{cd3cd28icam2+psit}\\5: \texttt{cd3cd28icam2+u0126}\\6: \texttt{cd3cd28icam2+ly}\\7: \texttt{cd3cd28icam2}\\8: \texttt{cd3cd28+aktinhib}\\9: \texttt{cd3cd28+g0076}\\10: \texttt{cd3cd28+psitect}\\11: \texttt{cd3cd28+u0126}\\12: \texttt{cd3cd28+ly}\\13: \texttt{pma}\\14: \texttt{b2camp}
  }
  \label{fig:real2}
\end{figure}

\section{Conclusions}
We have presented a structure learning method for BAPs, which can be viewed as a generalization of Gaussian linear DAG models that allow for certain latent variables. Our method is computationally feasible and the first of its kind. The results on simulated data are promising, keeping in mind that structure learning and inferring causal effects are difficult, even for the easier case with DAGs. The main sources of errors (given the model assumptions are fulfilled) are sampling variability, finding a local optimum only, and not knowing the equivalence classes. Local optima are a general weakness of many structure learning methods in Bayesian networks since this problem is NP-hard in general \citep{ ChiD96}. In our simulations, overestimating the equivalence class leads to too few causal effects, while the opposite happens if we underestimate it. On the other hand, our approach of greedily approximating the empirical equivalence class is building on the idea that some models are statistically indistinguishable, due to limited sample size and estimation error. Therefore, our approach has the advantage that it can include neighboring equivalence classes, that score almost as well, which is desirable from a statistical point of view. Our theoretical results about model equivalence go some way towards characterizing the distributional equivalence classes in BAP models and allow us to efficiently approximate them empirically.

In many applications, not all relevant variables are observed, calling for hidden variable models. While there have been structure learning methods for general hidden variable models for many years (FCI, RFCI, FCI+, see \citet{ SpiPA93, ColDA12, ClaTA13}), causal inference based on these models is very conservative \citep{MalDS17}. BAP models are restricted hidden variable models, where the restriction comes from the bow-freeness constaint. As such, they form an interesting middle ground between general hidden variable models and models that do not allow any hidden variables. In particular, the bow-freeness constraint leads to improved identifiability of causal effects when compared to general hidden variable models, while being more conservative than models without hidden variables. This makes our structure learning algorithm for BAPs a useful addition to existing structure learning methods. Structure learning for a different type of restricted hidden variable models is considered in \citet{FroBA17}, and it will be interesting to compare our results with this method.

\appendix
\section{Appendix}
\subsection{Distributional Equivalence}
\subsubsection{Necessary Conditions}
The following Lemma shows that the point $(B,\Omega) = (0,I)$ is non-singular for the map $\phi: \Theta_G \rightarrow \mathcal{S}_G$ for any BAP $G$.
The result also appears in \citet{BriCP02} and \citet{DrtMA11}.

\begin{lemma} \label{lem:ident}
Let $G$ be a BAP with parameters $(B, \Omega)$, and let $\phi$ be as in (\ref{eq:covmat}).  Then $\phi^{-1}(I) = \{(0, I)\}$; that is,
the parameters are uniquely identifiable at $\Sigma = I$ (or indeed at any diagonal $\Sigma$).
\end{lemma}

\begin{proof}
We proceed by induction on $d$, the number of vertices in $G$. If $d=1$ then the result is trivial since $B=0$.

Otherwise assume without loss of generality that the last vertex $d$ has no children. The result holds for the subgraph of the remaining vertices by the induction hypothesis and the fact that the distribution of $\bf X$ is defined recursively. We know that $\Sigma$ is of the form
\begin{align*}
  \Sigma &= (I_d-B)^{-1} \Omega (I_d-B)^{-T},
\end{align*}
and we may deduce by the induction hypothesis that the first $d-1$ rows of $B$ are zero, and the upper $(d-1) \times (d-1)$-sub matrix of $\Omega$ is the identity matrix. Hence
\begin{align*}
  \Omega &= \left( \begin{array}{cc} 
  I_{d-1} & \bs\omega\\
  \bs\omega^T & \omega_{dd}\\
  \end{array}\right),
\end{align*}
where $\bs\omega^T = (\omega_{1d}, \ldots, \omega_{d-1,d})$, and
\begin{align*}
  (I_d-B)^{-1} = 
  \left( \begin{array}{cc} 
  I_{d-1} & 0\\
  -\bs\beta^T & 1\\
  \end{array}\right)^{-1} = 
  \left( \begin{array}{cc} 
  I_{d-1} & 0\\
  \bs\beta^T & 1\\
  \end{array}\right),
\end{align*}
where $\bs\beta^T = (\beta_{d1}, \ldots, \beta_{d,d-1})$. Hence 
\begin{align*}
  \Sigma &= (I-B)^{-1} \Omega (I-B)^{-T}\\
  &= \left( \begin{array}{cc} 
  I_{d-1} & 0\\
  \bs\beta^T & 1\\
  \end{array}\right)
  \left( \begin{array}{cc} 
  I_{d-1} & \bs\omega\\
  \bs\omega^T & \omega_{dd}\\
  \end{array}\right)
  \left( \begin{array}{cc} 
  I_{d-1} & \bs\beta\\
  0 & 1\\
  \end{array}\right)
  \\
  &= \left( \begin{array}{cc} 
  I_{d-1} & 0\\
  \bs\beta^T & 1\\
  \end{array}\right)
  \left( \begin{array}{cc} 
  I_{d-1} & \bs\beta+\bs\omega\\
  \bs\omega^T & \bs\beta^T\bs\omega + \omega_{dd}\\
  \end{array}\right)
  \\
  \intertext{but note that $\bs\beta^T\bs\omega = 0$ by the bow-free assumption, so we get}
  \Sigma &= 
   \left( \begin{array}{cc} 
  I_{d-1} & \bs\beta + \bs\omega\\
  \bs\beta^T + \bs\omega^T & \| \bs\beta \|^2 + \omega_{dd}\\
  \end{array}\right),
\end{align*}
and hence $\bs\beta + \bs\omega = \bs 0$. Now note that for each $j$, either $\beta_{dj} = 0$ or $\omega_{jd} = 0$ by the bow-free assumption; hence $\bs\beta + \bs\omega = \bs 0$ implies that $\bs\beta = \bs\omega = \bs 0$, leaving
\begin{align*}
  \Sigma =  \left( \begin{array}{cc} 
  I_{d-1} & \bs 0\\
  \bs 0^T & \omega_{dd}\\
  \end{array}\right)
\end{align*}
and hence $\omega_{dd} = 1$. This completes the result.
\end{proof}

\begin{corollary} \label{cor:neig}
Let $G$ be a BAP.  For some neighborhood $U$ of the set of covariance matrices containing $I$, if $\Sigma \in \mathcal{S}_G \cap U$ with $\sigma_{ij} = 0$ (for $i \neq j$), then this implies that $\omega_{ij} = \beta_{ij} = 0$.
\end{corollary}

\begin{proof}
Since $\phi$ is nonsingular and differentiable at $\theta_0 = (O, I)$, its partial derivatives are defined and given by $\frac{\partial \phi}{\partial \omega_{ij}} (\theta_0) = 1$ and $\frac{\partial \phi}{\partial \beta{ij}} (\theta_0) = 1$ (this can be shown via a Taylor expansion for example). Therefore, in a small neighborhood around $\phi(\theta_0)$ we have $\sigma_{ij} = 0$ only if $\omega_{ij} = \omega_{ji} = \beta_{ij} = 0$.
\end{proof}

Note that Lemma~\ref{lem:ident} allows a direct proof of the fact that having the same skeleton is necessary for BAPs to be distributionally equivalent by looking at the tangent spaces of the models at $\Sigma = I$ and showing that they are determined by the skeletons of the graphs.

In the proof of Theorem~\ref{thm:subgraph} we make use of the language of polynomial varieties (see \citet{CoxDA07} for an overview). A variety is a set defined by the zeros of some collection of polynomials (in our case polynomials in the entries of $\Sigma$), and all SEM models are varieties.

Let $G$ be a BAP with vertices $V = W \dot\cup \overline{W}$ where $W \cap \overline{W} = \emptyset$. Let $\mathcal{B}_G^W$ be the set of matrices $B \in \mathcal{B}_G$ such that only entries corresponding to directed edges in $G$ between vertices in $W$ have non-zero coefficients. Similarly, let $\mathcal{O}_G^W$ be the set of $\Omega \in \mathcal{O}_G$ such that entries corresponding to edges outside $W$ are zero and diagonal entries outside $W$ are 1. 

Define a model $\tilde{\mathcal{S}}_G^W$ as the image of the map $\phi$ applied to $(\mathcal{B}_G^W, \mathcal{O}_G^W)$. So in other words, we only manipulate parameters in $G$ that correspond to vertices and edges in $G_W$.  The resulting model is canonically isomorphic to $\mathcal{S}_{G_W}$ via a simple projection, since this is the same setup as for the BAP $G_W$, but with the matrices extended to include independent vertices in $\overline{W} \equiv V \setminus W$.  

Let $\mathcal{T}_W$ be the set of covariance matrices $\Sigma$ on $V$ such that $\Sigma_{\overline{W}\overline{W}} = I$ and $\Sigma_{W \overline{W}} = 0$ (i.e.\ so that vertices outside $W$ are completely independent).

We will show that looking at the set of covariance matrices in $\mathcal{S}_G$ that are also in $\mathcal{T}_W$ is essentially the same as the set $\tilde{\mathcal{S}}_G^W$. Since the first set is a property of the full model, and the second set is determined by the subgraph $G_W$, this will be enough to prove Theorem~\ref{thm:subgraph}.

\begin{proof}[Proof of Theorem~\ref{thm:subgraph}]
What we need to prove is that $\mathcal{S}_G = \mathcal{S}_G'$ implies $\mathcal{S}_{G_W} = \mathcal{S}_{G'_W}$. Consider again the variety $\mathcal{T}_W$ defined above.

Clearly $\mathcal{S}_G = \mathcal{S}_{G'}$ implies $\mathcal{T}_W \cap \mathcal{S}_G = \mathcal{T}_W \cap \mathcal{S}_{G'}$. We will show that the irreducible component of $\mathcal{T}_W \cap \mathcal{S}_G$ which contains $\Sigma = I$ is the same as $\tilde{\mathcal{S}}_G^W$; since this last quantity is isomorphic to $\mathcal{S}_{G_W}$ we will prove the result.

First, note that $\tilde{\mathcal{S}}_G^W \subseteq \mathcal{T}_W$ and $\tilde{\mathcal{S}}_G^W \subseteq \mathcal{S}_{G}$, so clearly $\tilde{\mathcal{S}}_G^W \subseteq \mathcal{S}_{G} \cap \mathcal{T}_W$. In addition, note that by Corollary \ref{cor:neig}, in a neighborhood of $\Sigma = I$ every element of $\mathcal{T}_W \cap \mathcal{S}_G$ is also contained in $\tilde{\mathcal{S}}^W_G$. It follows that the entire irreducible component of $\mathcal{T}_W \cap \mathcal{S}_G$ containing $\Sigma = I$ is contained within $\tilde{\mathcal{S}}^W_G$, and therefore that the irreducible component of $\mathcal{T}_W \cap \mathcal{S}_G$ containing $\Sigma = I$ is $\tilde{\mathcal{S}}^W_G$.
\end{proof}

We can now prove the Theorem giving necessary conditions for BAP equivalence.

\begin{proof}[Proof of Corollary~\ref{cor:necessary}]
Let us first consider vertex pairs, i.e.\ $W = \{ i,j \}$. By Theorem~\ref{thm:subgraph} we have $G_W$ being distributionally equivalent to $G'_W$. If $G_W \ne G'_W$ we would have $i \ci_m j$ in one of the graphs but not the other, and using Lemma~\ref{lemma:m-sep} this would lead to a contradiction. Hence $G_W = G'_W$ for any vertex pair, and hence $G$ and $G'$ must have the same skeleton.

Let us now consider vertex triplets $W = \{ i,j,k \}$, such that (without loss of generality) there is a v-structure at $j$ in $G_W$. Then $i \ci_m k$ in $G_W$ and by the same argument as above we must have $i \ci_m k$ also in $G'_W$. This is only possible if there is a v-structure at $j$ in $G'_W$. Hence $G$ and $G'$ must have the same v-structures.
\end{proof}

\subsubsection{Sufficient Conditions}
We first make precise the definition of an important class of paths: \emph{treks}. These are paths that do not contain colliders. We adopt the notation of \citet{ FoyRA12}. A trek $\tau$ from $i$ to $j$ can have one of the following forms:
\begin{align*}
  v_l^L \leftarrow \cdots \leftarrow v_0^L \longleftrightarrow v_0^R \rightarrow \cdots \rightarrow v_r^R
\end{align*}
or
\begin{align*}
  v_l^L \leftarrow \cdots \leftarrow v_0 \rightarrow \cdots \rightarrow v_r^R,
\end{align*}
where $v_l^L = i$ and $v_r^R = j$ and in the second case $v_0 = v_0^L = v_0^R$. Accordingly, we define the \emph{left-hand side} of $\tau$ as $\mathrm{Left}(\tau) = v_l^L \leftarrow \cdots \leftarrow v_0^L$ and the \emph{right-hand side} of $\tau$ as $\mathrm{Right}(\tau) = v_0^R \rightarrow \cdots \rightarrow v_r^R$. Note that there is nothing inherently directional about a trek other that the (arbitrary) definition which end node is on the left. That is, every trek from $i$ to $j$ is also a trek from $j$ to $i$ just with the left and right sides switched. We denote the lengths of the left- and right-hand sides of a trek $\tau$ by $\lambda_L(\tau)$ and $\lambda_R(\tau)$ respectively. If $\tau$ does not contain a bidirected edge, we define its \emph{head} to be $H_\tau = v_0$. If the left- and right-hand sides of $\tau$ do not intersect (except possibly at $H_\tau$), we call $\tau$ \emph{simple}\footnote{Note that each side might well be self-intersecting, if the corresponding graph is cyclic.}. We define the following sets that will be useful later:
\begin{align*}
  \mathcal{D}_G^{ij} &= \{ \pi ~|~ \pi \mathrm{~is~a~directed~path~from~} i \mathrm{~to~} j \mathrm{~in~} G \}, \\
  \mathcal{T}_G^{ij} &= \{ \tau ~|~ \tau \mathrm{~is~a~trek~from~} i \mathrm{~to~} j \mathrm{~in~} G \}, \\
  \mathcal{S}_G^{ij} &= \{ \tau ~|~ \tau \mathrm{~is~a~simple~trek~from~} i \mathrm{~to~} j \mathrm{~in~} G \}.
\end{align*}
We will usually drop the subscript if it is clear from context which is the reference graph.

We now show some intermediate results that are well-known, but we prove them here for completeness nevertheless. All of these apply more generally to path diagrams (possibly cyclic and with bows).

\begin{lemma} \label{thm:geometric}
  Let $B \in \mathbb{R}^{d \times d}$ such that every eigenvalue $\lambda$ of $B$ satisfies $| \lambda | < 1$. Then $(I-B)^{-1}$ exists and is equal to $\sum_{s=0}^\infty B^s$.
\end{lemma}
\begin{proof}
  First note that $\det (\lambda I - B) = 0$ only if $| \lambda | < 1$, hence $\det (I-B) \ne 0$, and therefore $(I-B)^{-1}$ exists. The eigenvalue condition also implies $\lim_{l \rightarrow \infty} B^l = 0$, therefore
  \begin{align*}
    (I-B) \sum_{s=0}^\infty B^s = \lim_{l \rightarrow \infty} \sum_{s=0}^l ( B^s - B^{s+1} ) = \lim_{l \rightarrow \infty} (I - B^{l+1}) = I,
  \end{align*}
  and the result follows.
\end{proof}

\begin{lemma} \label{thm:path}
  Let $G$ be a path diagram over $d$ nodes and $B \in \mathcal{B}_G$. Then
  \begin{align*}
    (B^l)_{ij} = \sum_{\substack{\pi \in \mathcal{D}^{ji} \\ \lambda (\pi) = l}} \prod_{s \rightarrow t \in \pi} B_{ts}.
  \end{align*}
\end{lemma}
\begin{proof}
  By induction on $l$. For $l=1$ the claim follows from the definition of $\mathcal{B}_G$. Using the inductive hypothesis we get
  \begin{align*}
    (B^l)_{ij} = (B B^{l-1})_{ij} = \sum_{k=1}^d B_{ik} (B^{l-1})_{kj} = \sum_{k=1}^d B_{ik} \sum_{\substack{\pi \in \mathcal{D}^{jk} \\ \lambda (\pi) = l-1}} \prod_{s \rightarrow t \in \pi} B_{ts},
  \end{align*}
  and the claim follows, since every directed path from $j$ to $i$ of length $l$ can be decomposed into a directed path $\pi$ of length $l-1$ from $j$ to some node $k$ and the edge $k \rightarrow i$.
\end{proof}

\begin{lemma}
  Let $G$ be an acyclic path diagram over $d$ nodes and $B \in \mathcal{B}_G$. Then $(I-B)^{-1} = I + B + \ldots + B^{d-1}$.
\end{lemma}
\begin{proof}
  Since $G$ is acyclic, there is an ordering of the nodes, such that $B$ is strictly lower triangular and hence all its eigenvalues are zero. Furthermore, the longest directed path in $G$ has length $d-1$. Therefore the result follows from Lemma~\ref{thm:geometric} and Lemma~\ref{thm:path}.
\end{proof}

The following theorem is a version of Wright's theorem that applies to non-standardized variables. It does not require a proper parametrization (in the sense that $\Omega$ needs to be positive definite). This result is probably known to experts, but we could not find a proof in the literature.

\begin{theorem} \label{thm:unstandardised}
  Let $G$ be a (possibly cyclic) path diagram over $d$ nodes, $B \in \mathcal{B}_G$, and $\Omega \in \mathbb{R}^{d \times d}$ such that $\Omega$ is symmetric (but not necessarily positive definite) and $\Omega_{ij}=0$ if $i \leftrightarrow j$ is not an edge in $G$. Then the entries of the matrix $\phi = (I-B)^{-1} \Omega (I-B)^{-T}$ are given by
  \begin{align*}
    \phi_{ij} &= \sum_{\substack{\tau \in \mathcal{S}^{ij} \\ \leftrightarrow \in \tau}} \prod_{s \rightarrow t \in \tau} B_{ts} \prod_{s \leftrightarrow t \in \tau} \Omega_{st} + \sum_{\substack{\tau \in \mathcal{S}^{ij} \\ \leftrightarrow \notin \tau}} \prod_{s \rightarrow t \in \tau} B_{ts} \cdot \phi_{H_\tau H_\tau} \qquad (i \ne j), \\
    \phi_{ii} &= \sum_{\substack{\tau \in \mathcal{T}^{ii} \\ \leftrightarrow \in \tau}} \prod_{s \rightarrow t \in \tau} B_{ts} \prod_{s \leftrightarrow t \in \tau} \Omega_{st} + \sum_{\substack{\tau \in \mathcal{T}^{ii} \\ \leftrightarrow \notin \tau}} \prod_{s \rightarrow t \in \tau} B_{ts} \cdot \Omega_{H_\tau H_\tau} + \Omega_{ii}.
  \end{align*}
\end{theorem}
\begin{proof}
  Let us write
  \begin{align*}
    c_e(\tau; B, \Omega) = \prod_{s \rightarrow t \in \tau} B_{ts} \prod_{s \leftrightarrow t \in \tau} \Omega_{st}
  \end{align*}
  as a shorthand for the edge contribution\footnote{That is, the contribution depending only on the edge labels and not the diagonal elements of $\Omega$.} of a trek $\tau$ given parameter matrices $B$ and $\Omega$. We write $c(\tau; B, \Omega) = c_e(\tau; B, \Omega) \cdot \Omega_{H_\tau H_\tau}$ for the total contribution of $\tau$ (where we define $\Omega_{H_\tau H_\tau}$ to be $1$ if $\tau$ contains a bidirected edge and therefore $H_\tau=\varnothing$).

  Using Lemma~\ref{thm:geometric}, we can expand $\phi$ as $\phi = \sum_{k=0}^\infty \sum_{l=0}^\infty B^k \Omega (B^l)^T$. We now first show the following intermediate result, which interprets the entries of these matrices as contributions of certain treks:
  \begin{align} \label{eq:summand}
    ( B^k \Omega (B^l)^T )_{ij} = \sum_{\substack{\tau \in \mathcal{T}^{ij} \\ \lambda_L(\tau) = k \\ \lambda_R(\tau) = l}} c(\tau; B, \Omega) + \Omega_{ii} \mathbbm{1}\{i=j\},
  \end{align}
  for integers $k \ge 0,~ l \ge 0$. To see this, we expand the double matrix product and use Lemma~\ref{thm:path} to get
  \begin{align*}
    ( B^k \Omega (B^l)^T )_{ij} &= \sum_{a=1}^d \sum_{b=1}^d (B^k)_{ia} \Omega_{ab} (B^l)_{jb} \\
    &= \sum_{a=1}^d \sum_{b=1}^d \left( \sum_{\substack{\pi \in \mathcal{D}^{ai} \\ \lambda(\pi) = k}} \prod_{s \rightarrow t \in \pi} B_{ts} \right) \Omega_{ab} \left( \sum_{\substack{\pi \in \mathcal{D}^{bj} \\ \lambda(\pi) = l}} \prod_{s \rightarrow t \in \pi} B_{t s} \right),
  \end{align*}
  and \eqref{eq:summand} follows since each bracketed expression corresponds to one side of the trek from $i$ to $j$ via $a$ and $b$ (and the diagonal entries of $\Omega$ do not correspond to a trek, so they are separate). Now summing over $k$ and $l$ gives the following
  \begin{align} \label{eq:nonsimple}
    \phi_{ij} = \sum_{\tau \in \mathcal{T}^{ij}} c(\tau; B, \Omega) + \Omega_{ii} \mathbbm{1} \{ i=j \},
  \end{align}
  which gives the result for the diagonal entries $\phi_{ii}$.

  For the off-diagonal entries $\phi_{ij}$, we can get a simpler expression involving only simple treks and the diagonal entries $\phi_{ii}$. Note that every trek $\tau$ can be uniquely decomposed into a simple part $\xi(\tau)$ and a (possibly empty) non-simple part $\rho(\tau)$ (we just split at the point, where the right- and the left-hand sides of $\tau$ first intersect). Since
  \begin{align*}
    c(\tau) = \begin{cases} c_e(\xi(\tau)) \cdot \Omega_{H_{\xi(\tau)} H_{\xi(\tau)}} & \mathrm{if~} \rho(\tau) = \varnothing \\ c_e(\xi(\tau)) \cdot c(\rho(\tau)) & \mathrm{otherwise}, \end{cases}
  \end{align*}
  (dropping the parameter matrices $B$ and $\Omega$ in our notation), we can factor out the contributions of the simple parts. Note that if the simple part $\xi(\tau)$ contains a bidirected edge, then $\rho(\tau)$ must be empty and $\Omega_{H_\xi(\tau) H_\xi(\tau)} = 1$. Hence~\eqref{eq:nonsimple} becomes
  \begin{align*}
    \phi_{ij} &= \sum_{\substack{\tau \in \mathcal{T}^{ij} \\ \leftrightarrow \in \xi(\tau)}} c(\tau) + \sum_{\substack{\tau \in \mathcal{T}^{ij} \\ \leftrightarrow \notin \xi(\tau) \\ \rho(\tau) \ne \varnothing}} c(\tau) + \sum_{\substack{\tau \in \mathcal{T}^{ij} \\ \leftrightarrow \notin \xi(\tau) \\ \rho(\tau) = \varnothing}} c(\tau) \\
    &= \sum_{\substack{\tau \in \mathcal{T}^{ij} \\ \leftrightarrow \in \xi(\tau)}} c_e (\xi(\tau)) + \sum_{\substack{\tau \in \mathcal{T}^{ij} \\ \leftrightarrow \notin \xi(\tau) \\ \rho(\tau) \ne \varnothing}} c_e(\xi(\tau)) \cdot c(\rho(\tau)) + \sum_{\substack{\tau \in \mathcal{T}^{ij} \\ \leftrightarrow \notin \xi(\tau) \\ \rho(\tau) = \varnothing}} c_e(\xi(\tau)) \cdot \Omega_{H_{\xi(\tau)} H_{\xi(\tau)}} \\
    &= \sum_{\substack{\xi \in \mathcal{S}^{ij} \\ \leftrightarrow \in \xi}} c_e (\xi) + \sum_{\substack{\xi \in \mathcal{S}^{ij} \\ \leftrightarrow \notin \xi}} c_e(\xi) \sum_{\rho \in \mathcal{T}^{H_\xi H_\xi}} c(\rho) + \sum_{\substack{\xi \in \mathcal{S}^{ij} \\ \leftrightarrow \notin \xi}} c_e(\xi) \cdot \Omega_{H_\xi H_\xi} \\
    &= \sum_{\substack{\xi \in \mathcal{S}^{ij} \\ \leftrightarrow \in \xi}} c_e (\xi) + \sum_{\substack{\xi \in \mathcal{S}^{ij} \\ \leftrightarrow \notin \xi}} c_e(\xi) \left( \sum_{\rho \in \mathcal{T}^{H_\xi H_\xi}} c(\rho) + \Omega_{H_\xi H_\xi} \right),
  \end{align*}
  and the result follows.
\end{proof}

The following version for standardized parameters is often quoted as Wright's theorem.
\begin{theorem} \label{thm:wright}
  Let $G$ be a (not necessarily acyclic) path diagram over $d$ nodes, $B \in \mathcal{B}_G$, and $\Omega \in \mathbb{R}^{d \times d}$ such that $\Omega$ is symmetric (but not necessarily positive definite) and $\Omega_{ij}=0$ if $i \leftrightarrow j$ is not an edge in $G$. Furthermore assume that we have standardized parameters $B, \Omega$ such that $(\phi(B,\Omega))_{ii} = 1$ for all $i$. Then the off-diagonal entries of $\phi(B,\Omega)$ are given by
  \begin{align*}
    (\phi(B, \Omega))_{ij} = \sum_{\tau \in \mathcal{S}^{ij}} \prod_{s \rightarrow t \in \tau} B_{ts} \prod_{s \leftrightarrow t \in \tau} \Omega_{st}.
  \end{align*}
\end{theorem}
\begin{proof}
  This is a direct consequence of Theorem~\ref{thm:unstandardised}.
\end{proof}

We can now prove Theorem~\ref{thm:sufficient}, which is a consequence of Wright's formula.
\begin{proof}[Proof of Theorem~\ref{thm:sufficient}]
  Let $\theta_{G_1} \in \bar{\Theta}_{G_1}$ and choose $\theta_{G_2} = (B_2, \Omega_2)$ such that their edge labels agree, that is,
  \begin{align*}
    (B_2)_{ij} = \begin{cases} (B_1)_{ij} & \mathrm{if~} i \leftarrow j \in G_1, i \leftarrow j \in G_2, \\ (B_1)_{ji} & \mathrm{if~} i \rightarrow j \in G_1, i \leftarrow j \in G_2, \\ (\Omega_1)_{ij} & \mathrm{if~} i \leftrightarrow j \in G_1, i \leftarrow j \in G_2, \\ 0 & \mathrm{if~} i \leftarrow j \notin G_2, \end{cases}
  \end{align*}
  and
  \begin{align*}
    (\Omega_2)_{ij} = \begin{cases} (B_1)_{ij} & \mathrm{if~} i \leftarrow j \in G_1, i \leftrightarrow j \in G_2, \\ (B_1)_{ji} & \mathrm{if~} i \rightarrow j \in G_1, i \leftrightarrow j \in G_2, \\ (\Omega_1)_{ij} & \mathrm{if~} i \leftrightarrow j \in G_1, i \leftrightarrow j \in G_2, \\ 0 & \mathrm{if~} i \leftrightarrow j \notin G_2. \end{cases}
  \end{align*}
  This is possible since $G_1$ and $G_2$ have the same skeleton: we just assign the edge labels of $G_1$ to $G_2$, irrespective of the edge type. The diagonal entries of $\Omega_2$ are still free---we now show that they can be used to enforce
  \begin{align} \label{eq:standardisation}
    (\phi(B_2, \Omega_2))_{ii} = 1
  \end{align}
  for all $i$, which defines a linear system for the diagonal entries of $\Omega_2$. Let $\mathbf{d} = \mathrm{diag}(\Omega_2)$ be the vector consisting of the diagonal elements of $\Omega_2$, and write~\eqref{eq:standardisation} as $M \mathbf{d} + \mathbf{c} = \mathbf{1}$, where $M$ is the coefficient matrix of the linear system, and $\mathbf{c}$ is constant. To show that~\eqref{eq:standardisation} always has a solution, we need to show that $\mathrm{det}(M) \ne 0$. Without loss of generality, assume that the nodes are topologically ordered according to $G_2$ (this is possible since $G_2$ is assumed to be acyclic), that is, there is no directed path from $i$ to $j$ if $i > j$. Then we have $H_\tau < i$ (or $H_\tau = \varnothing$) for all $\tau \in \mathcal{T}^{ii}$, and using the expression for $\phi_{ii}$ in Theorem~\ref{thm:unstandardised} we see that $M$ must be lower triangular with diagonal equal to $1$. Thus $\mathrm{det}(M) = 1$, and we can enforce~\eqref{eq:standardisation}.

  Since $G_1$ and $G_2$ share the same collider triples, the sets of simple treks between any two nodes are the same in both graphs: $\mathcal{S}_{G_1}^{ij} = \mathcal{S}_{G_2}^{ij}$ $\forall i,j$. Together with Theorem~\ref{thm:wright} and the fact that the edge labels agree this shows that
  \begin{align} \label{sigmaequal}
    \phi (\theta_{G_1}) = \phi (\theta_{G_2}).
  \end{align}

  What is left to show is that $\Omega_2$ is a valid covariance matrix, that is, it is positive semi-definite. By \eqref{eq:covmat} and \eqref{sigmaequal} we have that
  \begin{align*}
    \Omega_2 = (I-B_2) \Sigma_1 (I-B_2)^T,
  \end{align*}
  where $\Sigma_1 = \phi (\theta_{G_1})$. Since $\Sigma_1$ is positive semidefinite, so is $\Omega_2$.
\end{proof}

\subsection{Likelihood Separation}
\label{sec:likelihoodseparation}
Since we can write $\bm\epsilon = \bm\epsilon (\mathbf{X})$ as a function of $\mathbf{X} = (X_1, \ldots, X_d)$, we have that their densities satisfy
\begin{align} \label{eq:joint.density}
  p_G^{\mathbf{X}}(X_1, \ldots, X_d) = p_G^{\bm\epsilon} (\epsilon_1 (\mathbf{X}), \ldots, \epsilon_d (\mathbf{X})).
\end{align}
The joint density for the errors $\bm\epsilon$ can be factorized according to the independence structure implied by $\Omega$. Let us adopt the notation $\{X\}_I \coloneqq \{ X_i \}_{i \in I}$ and $\{\epsilon\}_I \coloneqq \{ \epsilon_i \}_{i \in I}$ for some index set $I$. Then we have $\{\epsilon\}_{C_k} \ci \{\epsilon\}_{C_l}$ $\forall k \ne l$. Furthermore, we implicitly refer to marginalized densities via the arguments, i.e.\ we write $p_G^{\bm\epsilon} ( \{\epsilon\}_{C_1} )$ for the marginal density of $\{\epsilon\}_{C_1}$. We can thus write
\begin{align*}
  p_G^{\bm\epsilon} (\epsilon_1, \ldots, \epsilon_d) = p_G^{\bm\epsilon} ( \{\epsilon\}_{C_1} ) \cdots p_G^{\bm\epsilon} ( \{\epsilon\}_{C_K} ).
\end{align*}
Hence \eqref{eq:joint.density} becomes
\begin{align} \label{eq:joint.density2}
  p_G^{\mathbf{X}}(X_1, \ldots, X_d) = \prod_k p_G^{\bm\epsilon} \bigg( \Big\{ X_i - \sum_{j \in \mathrm{pa}(i)} B_{ij} X_j \Big\}_{i \in C_k} \bigg).
\end{align}
Each factor depends only on the nodes in the respective component $C_k$ and the parents of that component $\mathrm{pa}(C_k)$. By the same argument the joint density of the submodel $G_k$ is
\begin{align*}
  p_{G_k}^{\mathbf{X}} ( \{X\}_{V_k} ) &= p_{G_k}^{\bm\epsilon} ( \{\epsilon\}_{C_k} ) \prod_{j \in \mathrm{pa}(C_k) \setminus C_k} p_{G_k}^{\bm\epsilon} ( \epsilon_j ) \\
  &= p_{G_k}^{\bm\epsilon} \bigg( \Big\{ X_i - \sum_{j \in \mathrm{pa}(i)} B_{ij} X_j \Big\}_{i \in C_k} \bigg) \prod_{j \in \mathrm{pa}(C_k) \setminus C_k} p_{G_k}^{\bm\epsilon} ( X_j ).
\end{align*}
This factorization is symbolic, since the parents $\{ X_j \}_{j \in \mathrm{pa}(C_k) \setminus C_k}$ will not be independent in general. This does not matter, however, since these terms cancel when reconstructing the full density $p_G^X(X_1, \ldots, X_d)$ later. The advantage of this symbolic factorization is that we can still \emph{fit} the (wrong) submodel and then use the easier to compute product of marginal densities to reconstruct the full density, rather than doing the same with the actual submodel factorization and the joint density of the component parents.

Note that
\begin{align*}
  p_{G_k}^{\bm\epsilon} \bigg( \Big\{ X_i - \sum_{j \in \mathrm{pa}(i)} B_{ij} X_j \Big\}_{i \in C_k} \bigg) = p_{G}^{\bm\epsilon} \bigg( \Big\{ X_i - \sum_{j \in \mathrm{pa}(i)} B_{ij} X_j \Big\}_{i \in C_k} \bigg),
\end{align*}
that is, the conditionals $\{ X \}_{C_k} | \{ X \}_{\mathrm{pa}(C_k) \setminus C_k}$ are the same in models $G$ and $G_k$. This is because the structural equations of $\{X\}_{C_k}$ are the same in these models. Note also that $p_{G_k}^{\bm\epsilon} ( X_j ) = p_{G_k}^{\mathbf{X}} ( X_j )$ for all $j \in \mathrm{pa}(C_k) \setminus C_k$ and all $k$, since $\mathrm{pa}(C_k) \setminus C_k$ are source nodes in model $G_k$ (all edges between them were removed).

Thus we can reconstruct the full joint density~\eqref{eq:joint.density2} from joint densities of the connected component submodels and marginal densities of the parent variables:
\begin{align*}
  p_G^{\mathbf{X}}(X_1, \ldots, X_d) = \prod_k p_{G_k}^{\mathbf{X}} ( \{X\}_{V_k} ) \cdot \left( \prod_{j \in \mathrm{pa}(C_k) \setminus C_k} p_{G_k}^{\mathbf{X}} ( X_j ) \right)^{-1}.
\end{align*}
Writing $\mathcal{D}$ for the observed data $\{ x_i^s \}$ (with $1 \le i \le d$ and $1 \le s \le n$), the log-likelihood can then be written as
\begin{align*}
  l (p_G^{\mathbf{X}}; \mathcal{D}) &= \sum_{s=1}^n \log p_G^{\mathbf{X}} (x_1^{(s)}, \ldots, x_p^{(s)}) \notag\\
  &= \sum_{s=1}^n \sum_k \left( \log p_{G_k}^{\mathbf{X}} (\{x_i^{(s)}\}_{i \in V_k}) - \sum_{j \in \mathrm{pa}(C_k) \setminus C_k} \log p_{G_k}^{\mathbf{X}} (x_j^{(s)}) \right) \notag\\
  &= \sum_k \left( l (p_{G_k}^{\mathbf{X}}; \{x_i^{(s)}\}^{s=1,\ldots,n}_{i \in V_k}) - \sum_{j \in \mathrm{pa}(C_k) \setminus C_k} l (p_{G_k}^{\mathbf{X}}; \{x_j^{(s)}\}^{s=1,\ldots,n}) \right),
\end{align*}
where $l (p_{G_k}^{\mathbf{X}}; \{x_j^{(s)}\}^{s=1,\ldots,n})$ refers to the likelihood of the $X_j$-marginal of $p_{G_k}^{\mathbf{X}}$.

\subsection{Symmetry and Irreducibility of Markov Chain}
\label{sec:mcmc}
We show that the transition matrix of the Markov Chain described in Algorithm~\ref{alg:mcmc} is symmetric and irreducible. For two BAPs $G, G'$, let $P(G, G')$ be the probability of a single step transition from $G$ to $G'$.

\begin{theorem}
We have
\begin{enumerate}
  \item Symmetry: $P(G, G') = P(G', G)$.
  \item Irreducibility: $\exists G_1, \ldots, G_n$ such that
  \begin{align*}
    P(G, G_1) \left( \prod_{i = 1}^{n-1} P(G_i, G_{i+1}) \right) P(G_{n-1}, G') > 0.
  \end{align*}
\end{enumerate}
\end{theorem}
\begin{proof}
Let $p$ be the probability of sampling one position $(i,j)$, i.e.\ $p = 1/(d(d-1))$. Let us first consider the case where $G$ and $G'$ only differ by one edge addition, i.e.\ WLOG either
\begin{itemize}
  \item
  $G = G' \cup i \rightarrow j$.
  \item
  $G = G' \cup i \leftrightarrow j$.
\end{itemize}
In both cases we get $P(G, G') = p/2 = P(G', G)$. In the first case, by multiplying the probabilities along the branches of~\ref{edge:remove} and~\ref{noedge:directed} in Algorithm~\ref{alg:mcmc} respectively, and since $G$ has no cycles. In the second case, by multiplying the probabilities along the branches of~\ref{edge:remove} and~\ref{noedge:bidirected} respectively. Hence symmetry holds, and irreducibility is trivially true in this case.

For the general case, note that the transitions described in Algorithm~\ref{alg:mcmc} involve either edge additions or deletions, so if $G, G'$ do not differ by only one edge addition, we have $P(G, G') = P(G', G) = 0$. Furthermore, we can always find a collection of graphs $G_1, \ldots, G_n$, such that irreducibility holds, e.g.\ by successively removing edges from $G$ until the graph is empty and then successively adding edges until we arrive at $G'$. Then we have $P(G_i, G_{i+1}) = p/2 > 0$ for all $1 \le i < n$ by the case considered above, and the claim follows.
\end{proof}

\bibliography{slbap}

\end{document}